\documentclass[12pt]{article}
\usepackage[margin=1in]{geometry} 
\usepackage{amsmath,amssymb}
\usepackage{amsthm}
\usepackage{ mathrsfs }
\usepackage{ bm }
\usepackage{enumitem}
\usepackage{mathtools}
\usepackage[linktocpage]{hyperref}
\hypersetup{
           breaklinks=true,   
}
\usepackage{authblk}

\newcommand\inv{^{-1}}

\newcommand{\N}{\mathbb{N}}
\newcommand{\Z}{\mathbb{Z}}
\newcommand{\R}{\mathbb{R}}
\newcommand{\C}{\mathbb{C}}

\newcommand{\PP}{\mathbb{P}}

\DeclareMathOperator{\sign}{sign}

\DeclareMathOperator{\id}{id}
\DeclareMathOperator{\GL}{GL}
\DeclareMathOperator{\Aut}{Aut}
\DeclareMathOperator{\End}{End}

\DeclareMathOperator{\set}{set}

\newcommand{\ind}[1]{\ensuremath{{\mathbf{1}\left\{#1\right\}}}}

\newcommand{\asto}{\overset{\text{a.s.}}{\to}}

\newcommand\restr[2]{{
  \left.\kern-\nulldelimiterspace 
  #1 
  \vphantom{\big|} 
  \right|_{#2} 
  }}

\newcommand{\inner}[2]{\ensuremath{\left\langle #1, #2\right\rangle}}
\newcommand{\gen}[1]{\ensuremath{\left\langle #1 \right\rangle}}

\newtheorem{theorem}{Theorem}[section]
\newtheorem{proposition}[theorem]{Proposition}
\newtheorem{corollary}[theorem]{Corollary}
\newtheorem{lemma}[theorem]{Lemma}
\newtheorem{definition}[theorem]{Definition}
\theoremstyle{definition}
\newtheorem{remark}[theorem]{Remark}
\newtheorem{examplex}[theorem]{Example}
\newenvironment{example}
  {\pushQED{\qed}\examplex}
  {\popQED\endexamplex}

\usepackage[
backend=biber,
style=alphabetic,
maxcitenames=2,
maxbibnames=99,
uniquelist=false,
backref=true,
hyperref=true
]{biblatex}
\usepackage{xurl}

\title{A tradeoff between universality of equivariant models and learnability of symmetries}
\author{Vasco Portilheiro}
\date{}

\begin{document}

\maketitle

\begin{abstract}
We prove an impossibility result, which in the context of function learning says the following: under certain conditions, it is impossible to simultaneously learn symmetries and functions equivariant under them using an ansatz consisting of equivariant functions. To formalize this statement, we carefully study notions of approximation for groups and semigroups. We analyze certain families of neural networks for whether they satisfy the conditions of the impossibility result: what we call ``linearly equivariant'' networks, and group-convolutional networks. A lot can be said precisely about linearly equivariant networks, making them theoretically useful. On the practical side, our analysis of group-convolutional neural networks allows us generalize the well-known ``convolution is all you need'' theorem to non-homogeneous spaces. We additionally find an important difference between group convolution and semigroup convolution.
\end{abstract}

\tableofcontents

\section{Introduction}

\subsection{Motivation}
\label{sec-intro-motivation}
Symmetries govern much of the real world. Certain natural transformations of objects give rise to invariants (a molecule rotated in space maintains its chemical properties) or equivariants (the time shown on a numberless clock transforms with any rotation of the clock's face). In a function learning setting, we are often interested in equivariant functions $f:X\to Y$ for a group $G$ acting on both $X$ and $Y$, meaning that $f(g\cdot x) = g\cdot f(x)$ for all $g \in G$.\footnote{``Action'' and other standard algebraic terms used in the introduction are reviewed in Section \ref{sec-algebra}.} 

Identifying the symmetries relevant to a problem often leads to advances. It is by now folkloric that convolutional neural networks revolutionized machine learning on images in part by taking advantage of translation equivariance --- the work of \textcite{LeCun1989} being perhaps most known, though the idea dates at least to \textcite{Neocognitron}. Constraining networks to be equivariant has since brought success in the cases of permutations of sets \parencite{DeepSets}, spherical rotations \parencite{Cohen2018, Kondor2018}, and general Lorentz transformations \parencite{Bogatskiy2020}, just to name a few.

These successes have increasingly been used to motivate attempts to learn symmetries from data, rather than imposing them a priori (see discussion of related work below). Most recently, ambitions have turned to creating architectures capable of simultaneously learning a general group while learning its equivariant functions \parencite{Zhou2021,Dehmamy2021}, which can be understood as a natural application of meta-learning or multitask learning. As far as we know, however, no attempt has yet been made to understand when this task is well-posed. Loosely, the question is: \emph{can we learn the underlying symmetry in data when performing function learning?} That is, for an unknown ``population group'' $G$, can we in general both approximate its equivariant functions and learn the group itself. An easy ``no'' is the answer given data from a single equivariant function — it might have extraneous symmetries not shared by the other equivariant functions. One might then ask what can be done given \emph{all} of a group's equivariant functions, or at least all the learnable ones. One moral of the story we tell is that even given all equivariant functions, there are meaningful problems in identifying the group. In a sense, this is already apparent from a vignette on polynomial functions.
\begin{example}
\label{ex-invariant-polynomials}
An often useful move is to note that equivariance of polynomial functions is equivalent to invariance in a certain tensor product space.\footnote{See \cite[\textsection 4.1]{EMLP}, as well as \cite[Lemma 2.1]{Yarotsky2018} and the reference to \cite{Worfolk1994} therein.} We pass then to the world of polynomial invariants, classical since the time of Hilbert. Suppose two groups $G$ and $H$ acting on such a function space agree on orbits: $Gf = Hf$ for all $f$. It is then not hard to show that a function is $G$-invariant if and only if it is $H$-invariant. That is, given all the polynomial invariant functions of $G$, we can at most only identify it up to orbits. With continuity considerations, an analogous statement holds more broadly for those invariant functions which one can approximate by invariant polynomials.
\end{example}

Rather than polynomials, we are most interested in our question in the context of neural networks
$$f: V_1 \overset{\sigma_1 \circ L_1}{\to} V_2 \overset{\sigma_2 \circ L_2}{\to} \cdots \overset{\sigma_{n-1} \circ L_{n-1}}{\to} V_{n}$$
composed of linear maps $L_k$ and non-linearities $\sigma_k$, in which each layer $\sigma_k \circ L_k : V_k \to V_{k+1}$ is equivariant under a pair of actions of $G$. The equivariance of each layer may be achieved by following a ``universally equivariant non-linearity'' design pattern: restricting each $L_k$ to be equivariant under the specific group $G$, while letting the $\sigma_k$ be equivariant under any member of a large class of groups containing $G$. This simple way of guaranteeing the equivariance of $f$ appears to be the most common method for incorporating symmetries in neural networks; it in fact has a long history, \textcite[Theorem 2.4]{WoodShaweTaylor1996} having for example studied which non-linearities are compatible with certain finite groups.

This setting is not only ubiquitous, but also natural if one wants to learn the group $G$ and simultaneously learn equivariant functions in a way that represents symmetries ``intrinsically,'' which we mean in the sense of \textcite{Yarotsky2018}, who contrasts this with ``extrinsic'' symmetry achieved by averaging inputs and outputs over $G$. Unfortunately, the ability to flexibly learn $G$ amongst a large class of groups using the above design pattern greatly constrains the possible non-linearities. We see that in the most general cases, only trivial non-linearities are possible. However, even if this is not the case, having this class of groups be too large may lead to \emph{symmetry non-uniqueness}: the linear maps $L_k$ which commute with actions of $G$ generally commute with actions of another group $H$, which may be contained in a class of groups that is too general. Such situations lead to the impossibility result that is the subject of this work.


\subsection{Summary of results}
We are interested in understanding when one can learn both objects and their symmetries, using some ``symmetric ansatz'' --- for example when can one simultaneously learn equivariant functions and group actions on the input and output spaces, using equivariant neural networks. The goal of the current work is twofold: first to formalize what this means, and second to argue that care should be taken to ensure approaches have certain properties.

After laying down in Section \ref{sec-prelim} the groundwork in algebra and topology required to discuss learning symmetries by approximation, we prove in Section \ref{sec-main-result} our main theoretical result, which we sketch here for the case of equivariant neural networks. Our starting point, and a recurrent theme of our work, is what we call symmetry non-uniqueness.
\begin{definition}[informal]
\label{def-non-uniqueness-informal}
Let $F$ be a collection of equivariant neural networks, whose equivariance groups belong to some specified set $\Gamma$. A \emph{symmetry non-uniqueness} is group-valued map $H:\Gamma \to \Gamma$ such that each $G$-equivariant network $f \in F$ is also $H(G)$-equivariant.
\end{definition}
Example \ref{ex-invariant-polynomials} suggests that symmetry non-uniquenesses is related to the learnability of symmetries. Theorem \ref{thm-main-thm} formalizes this abstractly, in terms of convergence spaces. A concrete version of the result is a \emph{learnability tradeoff} between (i) universal approximation for equivariant models, and (ii) the ability to learn a group from its learnable equivariant functions (see Corollary \ref{cor-learnability-tradeoff}).
\begingroup
\def\thetheorem{\ref{thm-main-thm}}
\begin{theorem}[informal]
If a symmetry non-uniqueness with certain properties exists, it is impossible for both of the following to be true:
\begin{enumerate}[label=(\roman*)]
\item for any $G \in \Gamma$ and learnable $G$-equivariant function $f_*$, one can learn $f_*$ and $G$ simultaneously --- that is, there exist $f_i \in F$ with respective equivariance groups $G_i\in\Gamma$ such that we have simultaneous approximations by  convergence, $f_i \to f_*$ and $G_i \to G$,
\item one can identify $G$ by knowing which are its learnable equivariant functions.
\end{enumerate}
\end{theorem}
\addtocounter{theorem}{-1}
\endgroup

In Section \ref{sec-ansatzes} we demonstrate applications this theory, studying the symmetry non\hyphen{}uniquenesses of particular ansatzes of equivariant neural networks.
\begin{itemize}
\item \emph{Linearly equivariant networks}. We introduce the notion of networks in which only the linear maps determine equivariance. Such networks would be the most flexible among those following the ``universally equivariant non-linearity'' design pattern, in terms of which symmetries can be learned. However, this flexibility greatly restricts the non-linearities, making linearly equivariant networks of little practical use. Their utility is instead theoretical, as their symmetry non-uniquenesses can be described quite precisely, and they may thus inform analyses of related architectures. We use as a concrete example the equivariant multilayer perceptrons of \textcite{EMLP}.
\item \emph{Group-convolutional networks}. On one hand, we find that group convolutions do not suffer from symmetry non-uniqueness under certain natural conditions, lending theoretical support to recent empirical successes in learning symmetries \parencite{Zhou2021,Dehmamy2021}. On the other hand, if one instead uses semigroup convolution, a large class of non-uniquenesses may arise. In the course of our analysis, we obtain a simple generalization to non-homogeneous spaces of the well-known theorem of \textcite{KondorTrivedi2018}, which says certain equivariant linear functions must be convolutions.
\end{itemize}

Finally, in Section \ref{sec-max} we discuss symmetry non-uniquenesses that are \emph{maximal} in the right sense, characterizing when Theorem \ref{thm-main-thm} may apply.
\begin{definition}[informal]
\label{def-maximality-informal}
A symmetry non-uniqueness $H_*$ is \emph{maximal} if for any  non-uniqueness $H$, we have for each $G \in \Gamma$ that $H(G)$ is a subgroup of $H_*(G)$.
\end{definition}
We first show that some of the results on group\hyphen{}convolutional networks of the previous section exhibit a maximal symmetry non-uniqueness. The remainder of Section \ref{sec-max} describes in detail a maximal non-uniquenesses for linearly equivariant networks. In particular, we show that for linear maps between semisimple representations of a group, the non-uniqueness is given by passing to the group algebra, and conclude by showing this generalizes to all unitarizable representations of groups of type I. These include most groups of interest: all compact second countable Hausdorff groups, and all abelian locally compact second countable Hausdorff groups. These results on maximality are known in representation theory, but the recontextualization for machine learning is new.

\subsection{Related work}
We review relevant machine learning literature, first on the theory of equivariant neural networks, and second on learning symmetries from data. As the field is developing, it is unlikely our summary remains thorough for very long, if indeed it is to begin with. Nonetheless, we seek to provide some relevant history, as well as hit on key themes.

\subsubsection{Equivariant neural networks}
The broad context for our work is the current incarnation of the study of group equivariant networks, which owes much to the parallel lines of work of \textcite{KondorPhd} and \textcite{CohenPhd} on group convolution. The introduction of general group\hyphen{}convolutional networks by \textcite{Cohen2016} has been vindicated both by empirical success \parencite{Cohen2018} and by the proof of \textcite{KondorTrivedi2018} that equivariance of linear maps is equivalent to group convolution (as current machine learning parlance would have it, ``convolution is all you need'') when the input and output spaces are homogenous and the group compact. The attendant \emph{geometric} approach to machine learning, in which data is interpreted as a signal over a space with symmetries, has enjoyed enough success to become the subject of an upcoming book of \textcite{GDLBook}. We also recommend Cohen's thesis \parencite{CohenPhd} for a modern perspective; many relevant ideas also exist in older work, which can be found in a survey of \textcite{Wood1996}.

Note that while powerful, the fact that in some cases convolutions can represent all linear equivariant maps says nothing of universal approximation. Increasingly-general universal approximation theorems have been proved in the case of finite groups: first for permutation-invariant functions  \parencite{DeepSets}, then invariant functions for any finite group \parencite{Maron2019}, and finally for equivariant functions \parencite{Ravanbakhsh2020}. Universal approximation has also been achieved for particular classical Lie groups \parencite{Bogatskiy2020,Dym2021,Villar2021}. These works draw on Yarotsky's application of classical invariant theory \parencite{Yarotsky2018}. A series of works by Elesedy \parencite{Elesedy2021, Elesedy2021Kernels, Elesedy2022} complements these results, proving explicit generalization benefits of equivariance. 

In presenting a general, non-convolutional approach to constructing group-equivariant networks, \textcite[Appendix D]{EMLP} show certain networks are not universal approximators. In Section \ref{sec-emlp} we contextualize this observation within our theory.

While we show our main result applies not only to supervised learning but unsupervised learning as well, we have yet to see general equivariant density-estimation algorithms which may serve as the setting. There is a relatively new field of estimating symmetric distributions using normalizing flows \parencite{Kohler2020,Rezende2019,Storras2021}, with existing approaches avoiding symmetry non-uniqueness by restricting to certain kinds of symmetries. For example, \emph{linear} symmetries of a probability density must be area-preserving.

\subsubsection{Learning symmetries}
While much of the theory on equivariant neural networks is modern, the allure of learning groups from data is not new. \textcite{Rao1998} attempt to learn a one-parameter Lie group by comparing pairs of images $x_1$ and $x_2=g\cdot x_1$ (represented as vectors) differing by a ``small'' group element $g$. This allows linearizing the problem, estimating a single matrix $A$ --- the Lie algebra generator --- and for each pair a scalar $t$, such that $x_2 \approx (I+tA)x_1$. This approach presages several themes in modern attempts to learn symmetry from data. 

One such theme is the comparison of inputs related by a group element. \textcite{CohenMasters} use a similar setup to learn toroidal subgroups of orthogonal matrices. \textcite{Anselmi2019} develop a method for recovering orbits of finite groups from covariance matrices, based previous theory \parencite{Anselmi2014}. \textcite{Wetzel2020} compare the outputs of a non-equivariant neural network applied to each of a pair of transformed inputs to search for invariants. More relevantly for us, \textcite{Dehmamy2021} approximate group\hyphen{}convolutional layers in neural networks by linearization; the authors then propose to learn the Lie algebra generators parameterizing the linear layers by predicting the rotation angle between two images. Their apparent success is supported theoretically by our study of ``group-like'' semigroup convolution, which suggests no symmetry non-uniqueness arises in the approach.

A different set of approaches consists of attempts to learn the ``degree of symmetry'' in a dataset: the standard story is that ``6'' not being the same as ``9'' demonstrates the need for some approximate notion of symmetry. Most formalizations of this idea can be understood as generalizations of trying to learn the distribution of $t$ above, given a fixed symmetry $A$ --- a unifying theme being the averaging over $t$. \textcite{Romero2022} replace group\hyphen{}convolutional layers with ones averaging over a learned subset of the group. \textcite{Augerino} average outputs of a neural network over samples of transformed inputs using a learned distribution for $t$. This work surfaces a thematic problem in learning symmetries: enforcing them restricts the hypothesis class, and thus learning theory suggests a loss-minimizing model will assume no symmetry. The authors propose symmetry-incentivizing regularization as a solution. In similar architectures, \textcite{Schwoebel2022}, together with preceding works \parencite{vanderWilk2018,vanderOuderaa2021}, argue instead for maximizing marginal likelihood.

The observation that imposing symmetry runs counter to learning inspires other angles. \textcite{Mouli2021} take a formal causal approach to the idea that, given a fixed collection of candidate groups, a learning algorithm should be equivariant to \emph{all} except those that directly contradict the training data. A more flexible idea is given by \textcite{Zhou2021}: to meta-learn parameter-sharing schemes in linear layers --- shown by \textcite{Ravanbakhsh2017} to be equivalent to equivariance. We note that parameter-sharing schemes characterize the \emph{group algebra} rather than the group, but symmetry non-uniqueness is avoided by using non-linearities that are only equivariant under permutation groups. This can be seen as a special case of results on non-uniqueness in group\hyphen{}convolutional neural networks.

Learning Lie groups during density-estimation is also studied. An intuitive approach works well in basic cases, for which we recommend \textcite{Desai2022}. See also \textcite{Cahill2020}, who adapt local principal component analysis to learn a Lie algebra.

\section{Preliminaries}
\label{sec-prelim}
We present prerequisites in two broad areas: algebra and topology. As our intended audience may have quite varied background, we attempt to be as thorough as is reasonable, at the risk of presenting material some may consider basic. For the impatient reader, we provide guidance below as to what may be skipped, although moving directly to Section \ref{sec-main-result} is also possible. We do assume some ``mathematical maturity,'' for example referring in more technical passages to the notions of a Hilbert space as well a (Borel) measure space and its Lebesgue integral, without defining these below.
\subsection{Algebra}
\label{sec-algebra}
In Section \ref{sec-algebra-basics}, we recall definitions of homomorphisms and actions of semigroups, monoids, and groups, and finish by proving a useful result. We expect most readers are familiar with the definitions, and in this case  recommend a briefest glance at our discussion of actions and endomorphisms before moving on to Proposition \ref{prop-transitive-group}. Next, we introduce in Section \ref{sec-coupled-actions} some terminology of our own relating to equivariant functions. We discuss in Section \ref{sec-semigroup-algebra} the semigroup algebra associated to a semigroup. Section \ref{sec-rep-theory} reviews some representation theory, which the familiar reader may skim in order to get to Proposition \ref{prop-schur}.

\subsubsection{Basic definitions: semigroups and actions}
\label{sec-algebra-basics}
\begin{definition}
A \emph{semigroup} $(S,\cdot)$ is a set $S$ together with an associative binary operation:
$$s\cdot(t\cdot u) = (s\cdot t) \cdot u~~\forall s,t,u\in S.$$
A \emph{monoid} $(M,\cdot)$ is a semigroup with a (unique) element $\id \in M$, such that $\id\cdot s= s\cdot\id = s$ for all $s \in M$.
A \emph{group} $(G,\cdot)$ is a monoid such that for each $g \in G$ there exists a (unique) element $g\inv \in G$ such that $gg\inv = g\inv g = \id$.
\end{definition}
We often leave the binary operation implicit, referring for example to ``the semigroup $S$,'' and denoting products by $ab = a\cdot b$. If $ab = ba$ for all $a,b \in S$ then $S$ is called \emph{abelian}. Given a set of semigroup elements $A\subseteq S$, the \emph{semigroup generated by $A$}, denoted $\gen{A}$, is the smallest semigroup containing $A$; equivalently, it is the semigroup of products of finitely-many elements of $A$. A group generated by a set is defined similarly, allowing inversion of elements in the products.

\begin{definition}
A semigroup $S$ is a \emph{sub-semigroup} of a semigroup $T$, denoted $S\le T$, if $S \subseteq T$ as sets and the product on $S$ is the restriction of the product on $T$.

The \emph{direct product} of semigroups $S$ and $T$ is the semigroup $S\times T$ with product
$$(s_1, t_1) \cdot (s_2, t_2) = (s_1 s_2, t_1 t_2).$$

A \emph{semigroup homomorphism} $\phi:S\to T$ is a map such that $\phi(s\cdot t) = \phi(s) \cdot \phi(t)$. A \emph{monoid homomorphism} is a semigroup homomorphism between monoids such that $\phi(\id) = \id$. A \emph{group homomorphism} is a semigroup homomorphism between groups (and automatically a monoid homomorphism).
\end{definition}
The definitions of sub-semigroup and direct product extend naturally to monoids and groups. It is routine that the composition of homomorphisms is a homomorphism. The \emph{center} $Z(S)$ of a semigroup $S$ consists of all those elements which commute with all of $S$, and is itself a sub-semigroup of $S$.

For any object $X$, an \emph{endomorphism} is a ``morphism'' from $X$ to $X$, that is, a structure-preserving map. For example, an endomorphism of a vector space is a linear map, and an endomorphism of a topological space is a continuous map. The set of endomorphisms of $X$ forms a monoid which we denote by $\End(X)$. An \emph{automorphism} is an invertible endomorphism, and we write $\Aut(X)$ for the automorphism group of $X$. On the other hand, we may forget any structure on $X$ and treat is just as a set, written $\set(X)$; then $\End(\set(X))$ is the collection of \emph{all} functions from $X$ to itself, and $\Aut(\set(X))$ the subset of invertible functions. 
Semigroups, monoids, and groups can \emph{act on} a space $X$ if each of their elements is realized in $\End(X)$, as follows.
\begin{definition}
A \emph{(left) semigroup action} of a semigroup $S$ on an object $X$ is a semigroup homomorphism $\alpha:S\to \End(X)$. A \emph{monoid action} of a monoid $M$ on $X$ is a monoid homomorphism $\alpha:M\to \End(X)$. A \emph{group action} of a group $G$ on $X$ is a group homomorphism $\alpha:G\to \Aut(X)$. 
\end{definition}

We sometimes leave the homomorphism $\alpha$ implicit, saying ``$S$ acts on $X$'' to mean there exists an action $\alpha:S\to\End(X)$ and writing $s\cdot x=\alpha(s)(x)$. If $X$ is a vector space, and so $\End(X)$ consists of linear maps, it is standard to call actions \emph{semigroup representations}. We reserve the term \emph{representation} for linear group actions, when $\Aut(X) = \GL(X)$.

The \emph{orbit} of a point $x \in X$ under an action $\alpha$ of $S$ is the set $Sx = \{\alpha(s)(x):s\in S\}$.
 We call an action $\alpha:S\to\End(X)$ \emph{transitive} if there is $x_0 \in X$, such that for any $x \in X$ there exists $s_x \in S$ with $\alpha(s_x)(x_0) = x$. In this case $X$ is called a \emph{homogeneous space} for $S$.

A \emph{right action} of a semigroup $(S,\cdot)$ on a set $X$ is a left action of its \emph{opposite semigroup} $S^\text{opp}$, which is the semigroup $(S, *)$ where $s*t = t\cdot s$. In other words, such an action $\alpha:S^\text{opp}\to\End(X)$ is one that ``reverses multiplication'' in $S$ when acting on $X$:
$$\alpha(s\cdot t)(x) = \alpha(t* s)(x) = \alpha(t)(\alpha(s)(x)).$$

\begin{proposition}
\label{prop-transitive-group}
Let $S,T$ be semigroups, and $\phi:S\to T$ a homomorphism. If either of the following conditions holds
\begin{enumerate}[label=(\roman*)]
\item $\alpha,\beta$ are right actions on $X$ of $S,T$ respectively, $\alpha$ transitive,
\item $\alpha,\beta$ are left actions  on $X$ of $S,T$ respectively, $\alpha$ transitive, $S$ abelian, and $\phi(S)\leq Z(T)$,
\end{enumerate}
and $\alpha = \beta\circ \phi$, then $\alpha(S) = \beta(T)$.
\end{proposition}
\begin{proof}
The same proof goes through under either set of assumptions. First, $\alpha(S) \subseteq \beta(T)$, since for any $s\in S$ we have that $\alpha(s) = (\beta\circ\phi)(s) \in \beta(T)$. Consider on the other hand any $t \in T$. Fix some $x_0\in X$, and by the transitivity of $\alpha$ let $s\in S$ be such that $\alpha(s)(x_0) = \beta(t)(x_0).$ Then for any $x \in X$, picking $s_x \in S$ such that $\alpha(s_x)(x_0) = x$,
\begin{align*}
\beta(t)(x) &= \beta(t)(\alpha(s_x)(x_0))  = \beta(\phi(s_x)t)(x_0) \\
&= \alpha(s_x)(\beta(t)x_0) = \alpha(s s_x)(x_0) = \alpha(s)(x).
\end{align*}
So $\beta(t) = \alpha(s) \in \alpha(S)$ for any $t \in T$, and $\alpha(S) = \beta(T)$. 
\end{proof}

\subsubsection{Equivariant functions and coupled actions}
\label{sec-coupled-actions}
The following is the standard definition of equivariance.
\begin{definition}
Suppose a semigroup $S$ acts on both $X$ and $Y$. A function $f:X\to Y$ is \emph{equivariant under $S$}, or \emph{$S$-equivariant}, if
$$s \cdot f(x) = f(s\cdot x)$$
for any $x \in X$ and $s \in S$.
\end{definition}
Note that whether a function is $S$-equivariant depends on the actions of $S$ in question, a fact not emphasized by the above definition. The following two definitions are thus helpful.
\begin{definition}
A \emph{coupled action} $\alpha$ of $S$ on $X$ and $Y$ is a pair of actions, $\alpha_X$ of $S$ on $X$, and $\alpha_Y$ of $S$ on $Y$. We identify $\alpha$ with the action $\alpha: S \to \End(X)\times \End(Y)\subseteq \End(X\times Y)$ given by $\alpha(s)(x,y) = (\alpha_X(s)(x),\alpha_Y(s)(y))$.
\end{definition}
This notion is useful precisely because it pairs together transformations on $X$ and $Y$. In this spirit, we may write $\alpha(s) = (\alpha_X(s), \alpha_Y(s))$. When the actions are on vector spaces (and thus linear) we refer to $\alpha$ as a \emph{coupled representation}.
\begin{definition}
A function $f:X\to Y$ and coupled action $\alpha = (\alpha_X, \alpha_Y)$ of $S$ form an \emph{equivariant pair} $(f,\alpha)$ if
$$\alpha_Y(s) \circ f = f \circ \alpha_X(s)$$
for any $s \in S$. We also say that $f$ is \emph{$\alpha$-equivariant}.
\end{definition}

\subsubsection{The semigroup algebra}
\label{sec-semigroup-algebra}
The following construction is helpful in examples. We denote by $K$ a field, assumed to be of characteristic zero. (For us, $K$ might as well always be the reals $\R$ or complex numbers $\C$.)
\begin{definition}
Let $K$ be a field and $S$ a semigroup. The \emph{semigroup algebra} $K[S]$ of $S$ consists of elements $a \in K[S]$ that are formal sums
$$\sum_{s \in S} a(s) s$$
where $a(s) \in K$ is non-zero for at most finitely many $s \in S$. Equivalently, the elements are functions $a : S \to K$ that are non-zero only on finitely many elements of $S$. The scalar product of $k \in K$ and $a \in K[S]$ is given by $(ka)(s) = k \cdot a(s)$, and the sum of $a,b\in K[S]$ by $(a+b)(s) = a(s)+b(s)$. The product of two $a,b \in K[S]$ is given by
$$(ab)(s) = \sum_{t,u\in S : tu=s} a(t)b(u).$$
If $S$ is a group the above can be written as
$$(ab)(s) = \sum_{t\in S} a(s)b(t\inv s).$$
\end{definition}
A semigroup algebra is an \emph{algebra} in the sense that it is a ring satisfying certain properties \parencite[Chapter 10]{DummitFoote}. Group algebras are the most common instances of semigroup algebras. They are sometimes known as \emph{convolution algebras}, a term justified by the final expression above. One may verify that $K[S]$ is itself a semigroup under multiplication, so we may discuss its actions on a space $X$. Given a group $G$, the group algebra $K[G]$ is generally \emph{not} a group; the set of its invertible elements, called the \emph{group of units}, is denoted by $K[G]^\times$.
\begin{proposition}
Let $X$ be a vector space over $K$. The actions of a semigroup $S$ on $X$ are in bijection with the actions of its semigroup algebra $K[S]$ on $X$. The same holds for their actions on $\set(X)$.
\end{proposition}
\begin{proof}
We prove the result for actions on $X$, the proof for actions on $\set(X)$ being entirely similar. 
Consider then any action $\tilde{\alpha}:K[S] \to \End(X)$. There is a natural semigroup homomorphism $\iota:S\to K[S]$ given by inclusion. (More formally, for any $s\in S$ the function $\iota(s):S\to K$ maps $t$ to the multiplicative identity $1 \in K$ if $t=s$ and to the additive identity $0 \in K$ otherwise). Thus we obtain an action $\alpha:S\to \End(X)$ given by $\alpha = \tilde{\alpha}\circ \iota$.

Suppose on the other hand we have an action $\alpha:S\to \End(X)$. For any element $a \in K[S]$, define
$$\tilde{\alpha}(a) = \sum_{s \in S} a(s) \alpha(s).$$
One may verify this is a semigroup homomorphism:
\begin{align*}
\tilde{\alpha}(ab) 	= \sum_{s \in S} (ab)(s) \alpha(s) 
				= \sum_{s \in S} \sum_{tu=s} a(t) b(u) \alpha(t)\alpha(u) 
				= \sum_{t \in S} a(t)\alpha(t) \sum_{u \in S} b(u) \alpha(u) = \tilde{\alpha}(a)\tilde{\alpha}(b).
\end{align*}
\end{proof}
Give an action $\alpha$ of $S$ we refer to the above action $\tilde{\alpha}$ of $K[S]$ as the \emph{induced action} of the algebra. Note the above gives a bijection between \emph{group} actions of a group $G$ and \emph{semigroup} actions of its group algebra $K[G]$, which by restriction gives a bijection between the group actions of $G$ and the group actions of $K[G]^\times$.

The following easy observation serves as a starting point for studying non-uniqueness of symmetries.
\begin{proposition}
\label{prop-linear-equiv-algebra}
Let $X$ and $Y$ be vector spaces, and $\alpha=(\alpha_X, \alpha_Y)$ a coupled action of $S$ on $X$ and $Y$. A linear map $L : X\to Y$ is $\alpha$-equivariant if and only if it is $\tilde{\alpha}$-equivariant, where $\tilde{\alpha}=(\tilde{\alpha}_X, \tilde{\alpha}_Y)$ is the induced coupled action  of $K[S]$. The same holds for coupled actions on $\set(X)$ and $\set(Y)$.
\end{proposition}
\begin{proof}
If $L$ is $\tilde{\alpha}$-equivariant, by the inclusion of $S$ in $K[S]$ it is also $\alpha$-equivariant. On the other hand, if $(L,\alpha)$ is an equivariant pair then for any $a \in K[S]$
$$L \circ \sum_{s \in S} a(s) \alpha_X(s) =  \sum_{s \in S} a(s) (L\circ \alpha_X(s)) =  \sum_{s \in S} a(s) (\alpha_Y(s)\circ L),$$
applying first the linearity and then the $\alpha$-equivariance of $L$. The expression on the left is $L\circ \tilde{\alpha}(s)$ and on the right we have the definition of $\tilde{\alpha}(s)\circ L$, so $(L, \tilde{\alpha})$ is an equivariant pair. Noting that nowhere was the linearity of the actions used, we obtain the same result for actions on $\set(X)$ and $\set(Y)$.
\end{proof}

\subsubsection{Representation theory}
\label{sec-rep-theory}
We review irreducibility and versions of Schur's lemma for finite-dimensional and unitary representations, before characterizing the linear equivariant maps between semisimple representations satisfying a ``strong version'' of Schur's lemma in Proposition \ref{prop-schur}.

\begin{definition} For a representation $\rho:G\to\GL(V)$, a \emph{subrepresentation} is a representation $\sigma:G\to \GL(W)$ where $W$ is a subspace of $V$ and $\sigma(g)$ is the restriction of $\rho(g)$ to $W$ for each $g \in G$. A representation on $V\ne\{0\}$ is called \emph{irreducible} if its only subrepresentations are itself and the trivial representation on $\{0\}$.
\end{definition}
Note that a subrepresentation of a representation $\rho:G\to\GL(V)$ is given by any \emph{invariant subspace} $W$, that is, a subspace such that $\rho(g)w\in W$ for all $g\in G$ and $w\in W$. Irreducible representations have many properties which make them useful. For example, irreducible representations play nicely with group algebras. (Something like result below is generally called the Jacobson density theorem.)
\begin{proposition}[{\cite[Theorem 2.5]{Etingof2011}}]
\label{prop-density}
Let $\rho:G\to\GL(V)$ be a finite-dimensional irreducible representation. Then for an algebraically closed field $K$, the induced representation $\tilde{\rho}:K[G]\to\End(V)$ is surjective.
\end{proposition}

Perhaps the most famous statement about irreducible representations is the following result --- and variations of it --- often called Schur's lemma, which characterizes equivariant linear maps between irreducible representations. Such maps are also called \emph{intertwiners}.
\begin{proposition}[{\cite[Proposition 1.16]{Etingof2011}}]
Let $U$ and $V$ be vector spaces with respective irreducible representations $\rho$ and $\sigma$ of the same group $G$.
If $L:U\to V$ is a linear $(\rho,\sigma)$-equivariant map, it is either zero or an isomorphism.
\end{proposition}

For $U$ and $V$ finite-dimensional representations over an algebraically closed field (such as $\C$), the result above implies $L$ is equivariant if and only if it is a scalar multiple of the identity \parencite[Corollary 1.17]{Etingof2011}. We call this the \emph{strong version} of Schur's lemma. (We note this terminology is not standard.)

The strong version of Schur's lemma also holds for \emph{unitary} representations. Before stating it we review some basics. Recall that for a Hilbert space (complete inner product space) $V$, the \emph{adjoint} of a bounded linear operator $L\in\mathcal{L}(V)$ is the unique map $L^*$ such that $\inner{Lu}{v} = \inner{u}{L^*v}$ for all $u,v \in V$. A \emph{unitary} bounded linear operator $L\in\mathcal{L}(V)$ is one such that $LL^*=L^*L=I$, and we denote the set of such maps by $\mathcal{U}(V)$. With this background, we define unitary representations, noting that our definition of unitary representation below includes a continuity condition. (For a review of topology see the Section \ref{sec-topology}. A \emph{topological group} is a group with a topology such that multiplication and inversion are both continuous.) 
\begin{definition}
A \emph{unitary representation} of a topological group $G$ on a Hilbert space $V$ is a homomorphism $\pi:G\to\mathcal{U}(V)$ such that $(g,v)\mapsto\pi(g)v$ is continuous.
\end{definition}
Two unitary representations $\rho$ and $\sigma$ of $G$ are \emph{(unitarily) equivalent} if there exists a unitary $(\rho,\sigma)$-equivariant map $L$. Unitary representations satisfy the following strong version of Schur's lemma.
\begin{proposition}[{\cite[Proposition 1.A.11]{Bekka2019}}]
Let $V,V_1,V_2$ be complex Hilbert spaces.
\begin{enumerate}[label=(\roman*)]
\item A unitary representation $\pi:G\to\GL(V)$ is irreducible if and only if any $\pi$-equivariant linear $L:V\to V$ is a scalar multiple of the identity. 
\item If $\pi_1$ and $\pi_2$ are irreducible representations of $G$ on $V_1$ and $V_2$, then they are either equivalent, or no linear equivariant map exists between $V_1$ and $V_2$.
\end{enumerate}
\end{proposition}

Irreducible representations serve as building blocks for more general representations. One of the important constructions for combining them is the following.
\begin{definition}
For two representations $\rho:G\to\GL(U)$ and $\sigma:G\to\GL(V)$, the \emph{direct sum} is a representation $\rho\oplus\sigma:G\to \GL(U\oplus V)$, where
$$U\oplus V = \{(u,v): u \in U, v \in V\},$$
given by
$$(\rho\oplus\sigma)(g)(u,v) = (\rho(g)u, \sigma(g)v).$$
\end{definition}
Sometimes for convenience $(u,v)$ is written as $u\oplus v$. Similarly, for two linear maps $L_U : U_1 \to U_2$ and $L_V :V_1\to V_2$ one may write $L_U \oplus L_V$ for the map $(u,v)\mapsto (L_U u, L_V v)$. If the spaces in question are finite-dimensional, so $L_U$ and $L_V$ are matrices, the map $L_U \oplus L_V$ is given by a corresponding block-diagonal matrix. In general, we denote by $\bigoplus_{i \in I}\rho_i$ and $\bigoplus_{i \in I}V_i$ the direct sums of representations and vector spaces over an index set $I$. In this case, the elements of $\bigoplus_{i \in I}V_i$ are vectors for which only finitely many of the components $V_i$ are non-zero. We may sometimes omit the index set $I$ from our notation.

\begin{definition}
A representation $\rho:G \to \GL(V)$ is \emph{semisimple} if it is the direct sum of irreducible representations.
\end{definition}
We generally do not distinguish between semisimple representations and those representations which are isomorphic to semisimple ones, since results holding for the former hold for the latter passing through an appropriate isomorphism.

It is standard that any representation of a finite group is semisimple (see Maschke's and Wedderburn's Theorems, in particular \cite[Corollary 5, Chapter 18]{DummitFoote}). Semisimple representations lend themselves nicely to the study of equivariant maps. Indeed, we can easily characterize the equivariant linear maps between representations that are semisimple with component irreducible representations satisfying the strong version of Schur's lemma.
\begin{proposition}
\label{prop-schur}
Let $\rho_U:G\to\GL(U)$ and $\rho_V:G\to \GL(V)$ be semisimple representations, whose irreducible subrepresentations satisfy the strong version of Schur's lemma. Write $\rho_U = \bigoplus_{k} \rho_k^{\oplus n_k}$ and $\rho_V = \bigoplus_{k} \rho_k^{\oplus m_k}$ where $\rho_k:G\to\GL(W_k)$ are irreducible representations and $\rho_k^{\oplus n}$ denotes the direct sum of $n$ copies of $\rho_k$. A linear map $L:U\to V$ is $(\rho_U,\rho_V)$-equivariant if and only if it takes the form $\bigoplus_{k} L_k$ where each $L_k : W_k^{\oplus n_k} \to W_k^{\oplus m_k}$ is a linear map acting as a scalar on each copy of $W_k$ (i.e.\ a $m_k \times n_k$ ``matrix'' of scalar multiples of the identity).
\end{proposition}
Prior to the proof, we mention that a little bit of care is required in interpreting the expression $L=\bigoplus_k L_k$, due to the fact that $n_k$ or $m_k$ may be zero, potentially making the definition of $L_k$ unclear. Any confusion is dispelled by making the following observations. If $n_k=0$, meaning $U$ contains no copies of $W_k$, then the image of $L$ in the $W_k^{\oplus m_k}$ component of $V$ is identically zero; this can be seen by identifying $W_k^{\oplus n_k}$ with the trivial space $\{0\}$. Similarly if $m_k=0$ then $W_k^{\oplus n_k}$ in $U$ is in the kernel of $L$. See also \cite[Proposition 3.1.4]{Etingof2011} for a discussion of decompositions such as those of $U$ and $V$ above. For an illustration of the above result see Example \ref{ex-perm}.
\begin{proof}
Consider such a map $L=\bigoplus_{k} L_k$. To see it is equivariant it suffices to observe that each $L_k$ is. This in turn is true because for any $w_k =(w_{k,1}, \ldots, w_{k,n_k}) \in W_k^{\oplus n_k}$ we have that the $i$-th component of $L_k \rho_k^{\oplus n_k}(g)w_k$ for any $i \in [m_k]$ is given by
$$(L_k \rho_k^{\oplus n_k}(g)w_k)_i = \sum_{j = 1}^{ n_k} (L_k)_{ij} \rho_k(g)w_{k,j} = \rho^{\oplus n_k}_k(g)  \sum_{j =1}^{ n_k} (L_k)_{ij}  w_{k,j} = \rho^{\oplus n_k}_k(g) (L_k w_k)_i,$$
the second equality holding since $L_k$ acts on each copy of $W_k$ as a scalar.

Suppose on the other hand $L:U\to V$ is $(\rho_U,\rho_V)$-equivariant. By Schur's lemma (the non-strong version), $L=\bigoplus_{k} L_k$, for some $L_k : W_k^{\oplus n_k} \to W_k^{\oplus m_k}$. That each $L_k$ takes the form above is direct, by the assumption that the strong version of Schur's lemma holds.
\end{proof}

We conclude our review of representation theory by recalling two well-known facts (see for example \cite[Section 4.6]{Etingof2011}). First, any finite-dimensional unitary representation is semisimple, composed of unitary irreducible representations. Additionally, any finite-dimensional representation of a compact second countable Hausdorff topological group $G$ is \emph{unitarizable}, that is, isomorphic to a unitary representation.

\subsection{Convergence and topology}
\label{sec-topology}
We are interested in notions of approximation, both of functions and of symmetries, described by convergence. We thus describe the notion of a convergence space in Section \ref{sec-convergence-basics}. We relate this definition to more familiar topology in Section \ref{sec-topology-basics}. The latter may be skipped by a reader familiar with topology, though we recommend briefly reviewing the definition of and results on the compact-open topology presented towards the end (Proposition \ref{prop-compact-open-properties}).

 In applications, we are generally interested in learning a sub-semigroup of $\End(X)$ for some space $X$. The space of sub-semigroups of $\End(X)$ can be given various notions of convergence, built in turn from convergence structures on $\End(X)$ itself. We describe for each kind of convergence a property which makes their combination especially workable: what we call \emph{respecting limit elements} for the former, and \emph{admissibility} for the latter. In Section \ref{sec-admissibility} we describe the latter, providing the canonical example of convergence in the compact-open topology. We treat convergence of sub-semigroups in Section \ref{sec-hypertopology}, providing two examples of convergences which respect limit elements. We then show this property is exactly that of being stronger than lower Kuratowski-Painlev\'e convergence, which allows us to describe a general class of hit-and-miss topologies with the property.\footnote{This kind of property is also arises in non-topological notions of approximation for groups; see \cite{Alekseyev2001} and especially  \cite[Section 1.5]{Vershik1997}.}

In general, convergences with the above properties provide a broad class of examples to which the main result of Section \ref{sec-main-result}, Theorem \ref{thm-main-thm}, applies, thanks to Proposition \ref{prop-equivariance-closed}. The result also depends on the existence of a function continuous with respect to a convergence on the sub-semigroups of $\End(X)$. In Section \ref{sec-hypertopology} we highlight some cases in which continuity is easy to verify. As a whole, Section \ref{sec-hypertopology} is a bit technical and is meant to provide possible settings for Theorem \ref{thm-main-thm}. As such, the reader may wish to return to it after absorbing the context of Section \ref{sec-main-result}.

\subsubsection{Basic definitions: convergences spaces}
\label{sec-convergence-basics}
We follow \textcite{OBrien2021} and define convergence spaces in terms of nets. (It is shown in the cited work that this is equivalent to a more common definition in terms of filters, as in \cite{Dolecki2009}.) We recall the definition of a net, which generalizes the notion of a sequence.
\begin{definition}
A \emph{directed set} $I$ is a set with a relation $\le$, such that:
\begin{enumerate}[label=(\roman*)]
\item $\le$ is reflexive: $i\le i$ for all $i \in I$,
\item $\le$ is transitive: if $i\le j$ and $j\le k$ then $i \le k$,
\item for any $i,j \in I$ there exists $k\in I$ such that $i\le k$ and $j \le k$.
\end{enumerate}
A \emph{net} with values in $X$ is a function $x: I \to X$ from some directed set $I$ to $X$. We denote the net by $(x_i)_{i\in I}$ or simply $(x_i)$.
\end{definition}
A relation satisfying the first two points is called a \emph{preorder}. We sometimes make statements that quantify over nets with values in a certain set. This may appear problematic, as the collection of nets is generally not a set; we refer the reader to \cite[Theorem 2.1]{OBrien2021} for the resolution of the issue, which involves defining a set of ``nice'' nets $\mathfrak{N}(X)$.

There are several ways of defining the notion of a subnet. We use that introduced by \textcite{Aarnes1972}. (Note that in \cite{OBrien2021} the same notion is instead given the name \emph{quasi-subnet}.)
\begin{definition}
A net $(y_j)_{j \in J}$ is a \emph{subnet} of $(x_i)_{i\in I}$ if for every $i_0 \in I$ there exists a $j_0 \in J$ such that $\{y_j\}_{j \ge j_0} \subseteq \{x_i\}_{i \ge i_0}$; that is, for any $j \ge j_0$ there exists $i \ge i_0$ such that $y_j = x_i$.
\end{definition}
The notions of net and subnet are enough to define convergence.
\begin{definition}
A \emph{(proper)\footnote{When discussing general existing notions of convergence, we use the term ``proper'' to distinguish those which satisfy this definition.} convergence} on a set $X$ is a binary relation $\eta \subseteq \mathfrak{N}(X)\times X$, denoted by $(x_i) \overset{\eta}{\to} x$ or ``$(x_i)$ $\eta$-converges to $x$,'' satisfying the following axioms:
\begin{enumerate}[label=(\roman*)]
\item for any constant net, $(x_i)_{i \in I}$ with $x_i = x$ for all $i \in I$, we have $(x_i) \overset{\eta}{\to} x$, 
\item if $(x_i) \overset{\eta}{\to} x$, then $(y_j) \overset{\eta}{\to} x$ for any subnet $(y_j)$ of $(x_i)$,
\item if $(x_i)_{i \in I} \overset{\eta}{\to} x$ and $(y_i)_{i \in I} \overset{\eta}{\to} x$, and $z_i \in \{x_i, y_i\}$ for all $i \in I$, then $(z_i) \overset{\eta}{\to} x$.
\end{enumerate}
\end{definition}
 We note that a convergence can be extended from the ``nice'' nets $\mathfrak{N}(X)$ to all nets. We call a set $X$ with a convergence $\eta$ a \emph{convergence space}. When there is no risk of confusion, we may leave $\eta$ implicit, and write $(x_i) \to x$. Given two convergences $\eta_1$ and  $\eta_2$ on $X$, we call $\eta_1$ \emph{stronger} than $\eta_2$ if $(x_i)_{i \in I} \overset{\eta_1}{\to} x$ implies $(x_i)_{i \in I} \overset{\eta_2}{\to} x$. (Note that a convergence is trivially stronger than itself.)
 
Before moving on to topologies, we mention some analogues of familiar topological notions for convergences that are useful for us. Let $X$ be a convergence space. For a subset $A\subseteq X$, we write 
 $$\overline{A} = \{x \in X:\exists \text{ a net } (x_i)\to x,~x_i \in A\}.$$
We call $A$ \emph{closed} if $\overline{A}=A$. Note that $\overline{A}$ need not be closed in general (so we refrain from calling it the ``closure'' of $A$).

We call the convergence on $X$ \emph{Hausdorff} if $(x_i)\to x$ and $(x_i)\to x'$ implies $x=x'$. It is generally desirable that a convergence be Hausdorff, to ensure ``learning by approximation'' is a coherent notion.

Given another convergence space $Y$, the \emph{product convergence} is the proper convergence on $X\times Y$ given by $(x_i,y_i)_{i\in I} \to (x,y)$ when both $(x_i)\to x$ and $(y_i)\to y$.

Finally, we call a function $f:X\to Y$ such that $(x_i)\to x$ implies $(f(x_i)) \to f(x)$ \emph{continuous at $x$}, and denote by $C(X,Y)$ the functions that are continuous at every $x \in X$, called simply \emph{continuous}. Note that composition of two continuous functions is continuous, so $C(X) = C(X,X)$ is a monoid.

\subsubsection{Basic definitions: topological spaces}
\label{sec-topology-basics}
We recall that a \emph{topology} on a set $X$ is a collection $\tau$ of subsets of $X$, called \emph{open}, which contains the empty set and is closed under arbitrary unions and under finite intersections. A topology has a natural proper convergence, with $(x_i)_{i \in I} \to x$ when for any \emph{neighborhood} $U$ of $x$ (set containing an open set containing $x$), there is $i_0 \in I$ such that $x_i \in U$ for all $i\ge i_0$.  We call a set $K\subseteq X$ \emph{compact} if every net in $K$ has a convergent subnet. A subset $A\subseteq X$ is \emph{locally compact} if for any convergent net $(x_i)_{i \in I}$ in $A$ there is $i_0 \in I$ such that $\{x_i\}_{i\ge i_0}$ is contained in a compact subset of $A$. (These definitions work for convergence spaces, but we will only use them in the case of topologies). A topology $\tau_1$ is \emph{stronger} (or \emph{finer}) than a topology $\tau_2$ when $\tau_2 \subseteq \tau_1$. This coincides with the convergence in $\tau_1$ being stronger. Given two topologies, there is a unique \emph{supremum} of the two, the weakest topology stronger than both. The strongest topology is the \emph{discrete topology}, in which all sets are open.

In a topological space $(X,\tau)$ much importance is placed on separation properties. The topology on $X$ is called $T_0$ if for any distinct points $x$ and $y$ in $X$, there exists $U \in \tau$ which contains exactly one of $x$ and $y$. The topology is $T_1$ if for distinct $x,y \in X$ there exist respective neighborhoods $U_x$ and $U_y$ which do not contain $y$ and $x$, respectively. Finally (for our purposes), $X$ is $T_2$ or \emph{Hausdorff} if for any distinct $x,y \in X$ there exist disjoint respective neighborhoods, $U_x \cap U_y = \emptyset$. This coincides with the convergence in $\tau$ being Hausdorff. We call $\tau$ \emph{regular}, or $T_3$, if a point and disjoint closed set can be separated by disjoint open sets.

A \emph{base} is a collection $\mathcal{B}$ of subsets of $X$ such that for any $x \in X$ there is a $B \in \mathcal{B}$ containing $x$, and for any $B_1, B_2 \in \mathcal{B}$ both containing $x \in X$ there exists a set $B_3 \subseteq B_1 \cap B_2$ in $\mathcal{B}$ which contains $x$. A base generates a topology, in the sense that letting $\tau$ be the unions of sets in $\mathcal{B}$ makes $\tau$ a valid topology. A topology with countable base is \emph{second countable}. A collection of sets $\mathcal{B}$ is a \emph{subbase} of $\tau$ if its finite intersections are a base for $\tau$, or equivalently if $\tau$ is the weakest topology containing all the sets in $\mathcal{B}$. Subbases make it easy to verify the continuity of a function $f:X\to Y$. In particular, if $\mathcal{B}_Y$ is a subbase for $Y$, then $f$ is continuous if and only if $f\inv(B)$ is open in $X$ for every $B \in \mathcal{B}_Y$. Below we encounter topologies presented as the suprema of two other topologies. Continuity in these is also easy to check, since if $\tau_1$ has subbase $\mathcal{B}_1$ and $\tau_2$ has subbase $\mathcal{B}_2$ then $\mathcal{B}_1 \cup \mathcal{B}_2$ is a subbase for the supremum of $\tau_1$ and $\tau_2$. A \emph{homeomorphism} is a continuous bijection with continuous inverse, and is both \emph{closed} (takes closed sets to closed sets) and \emph{open} (takes open sets to open sets).

We recall a quite useful topology on spaces of continuous functions. 
\begin{definition}
The \emph{compact-open topology} on $C(X,Y)$ is the topology generated by the subbase of sets
$$V_{K,U} = \{f \in C(X,Y) : f(K) \subseteq U\}$$
where $K\subseteq X$ is compact and $U\subseteq X$ is open.
\end{definition}
Part of the utility of the compact-open topology comes from the following properties, compiled by \textcite{Arens1946}. Its perhaps most useful property --- admissibility --- is studied in the next section.
\begin{proposition}
\label{prop-compact-open-properties}
\item
\begin{enumerate}[label=(\roman*)]
\item The compact-open topology on $C(X,Y)$ is $T_n$ if $Y$ is, for each $n \in \{0,1,2,3\}$ (and similarly for the ``complete regularity'' axiom) \parencite[Theorem 1]{Arens1946}.
\item If $Y$ is a metric space, the compact-open topology on $C(X,Y)$ is that of uniform convergence on compact sets \parencite[Theorem 6]{Arens1946}
\item If $X$ is hemicompact and $Y$ metrizable, then the compact-open topology on $C(X,Y)$ is metrizable \parencite[Theorem 7]{Arens1946}.
\end{enumerate}
\end{proposition}
We recall that a net of functions $(f_i)$ from $X$ to a metric space $(Y,d)$ converges to $f$ uniformly on compact sets if for any compact $K\subseteq X$ we have
$$\sup_{x \in K} d(f_i(x), f(x)) \to 0.$$
Hemicompactness of $X$ is a technical condition, defined to hold when there exists a sequence of compact sets $K_n$ in $X$ such that for any compact $K\subseteq X$ there is some $n$ such that $K \subseteq K_n$. For example, $\R^n$ is hemicompact for any $n \in \N$ by the Heine-Borel theorem.

\subsubsection{Admissible topologies and topologies for semigroups}
\label{sec-admissibility}
Let $X$ and $Y$ be convergence spaces. The following notion of convergence turns out to be quite useful.
\begin{definition}
A net $(f_i)_{i\in I}$ of continuous functions $f_i: X\to Y$ \emph{converges continuously} to $f \in C(X,Y)$ if for any convergent net $(x_j)_{j\in J} \to x$ in $X$ we have $(f_i(x_j)) \to f(x)$. (The order on $I\times J$ is such that $(i_1, j_1) \le (i_2, j_2)$ if and only if $i_1 \le i_2$ and $j_1 \le j_2$.) A convergence on $C(X,Y)$ is called \emph{admissible} if it is stronger than continuous convergence.
\end{definition}
Our terminology follows that of Arens \parencite{Arens1946,Arens1951}, for whom an admissible topology is one whose convergence implies continuous convergence. Such topologies are also called \emph{conjoining topologies} \parencite{McCoy1988}. If the convergence is weaker than continuous convergence, the topology is called \emph{proper}, or \emph{splitting} (not to be confused with our definition of proper convergence). The following result marks the compact-open topology as distinctly useful.
\begin{proposition}[{\cite[Theorem 2]{Arens1946}}]
If $X$ is a locally compact Hausdorff topological space, then for any topological space $Y$ the compact-open topology on $C(X,Y)$ is both admissible and proper.
\end{proposition}

While convergence spaces are sufficient for our purposes, one may want to use topology to describe approximation of transformations of a space $X$, that is to place a topology on $\End(X)$, or some subspace of $\End(X)$ within which one wishes to be able to approximate. Furthermore, one might hope to describe topologically the approximation of sub-semigroups of $\End(X)$. In the most general case, $X$ is a topological space without any additional structure, so $\End(X) = C(X)$. As exhibited in the next section, ``good'' proper convergences are then available on the space of sub-semigroups of $\End(X)$, but ``good'' \emph{topological} convergences generally require $\End(X)$ to be locally compact; additionally, it is reasonble to want the topology on $\End(X)$ to be admissible (see Propositions \ref{prop-admissible-limit-semigroup} and \ref{prop-admissible-generators}). However, $\End(X)$ with the compact-open topology is not generally locally compact, even if $X$ is compact. Two natural approaches would be compactification and restricting to a better behaved subset of $\End(X)$. We do not engage with the former.\footnote{One interested in doing so might refer to \cite{Pestov1999,Uspenskij2001,Ruppert1984} (though these works are more precisely about compactifications of the group $\Aut(X)$).} The latter approach is quite flexible: for example if $X$ is a locally compact \emph{metric} space with compact closed balls, then its isometry group $\Aut(X)$ is locally compact \parencite[see the references in the introduction]{Abels2011}; also, the case that $X=\R^n$ as a vector space is broadly applicable to machine learning, and then $\End(X)\cong \R^{n^2}$ is of course locally compact, as is $\Aut(X)=\GL(\R^n)$.

We end by mentioning a different concern, which is that for a non-compact topological space $X$, the compact-open topology on $\Aut(X)$ may not make it a \emph{topological group} --- while multiplication (composition) is continuous, taking inverses may not be. For a locally compact space $X$ such that each point has a compact connected neighborhood, this concern can be dismissed by passing to the one-point compactification of $X$ \parencite{Dijkstra2005}.

\subsubsection{Topologies for spaces of sub-semigroups}
\label{sec-hypertopology}
We begin by examining two natural notions of convergence of semigroups, due to  \textcite[Chapter 9]{Thurston2022}. We note a useful property they share --- what we call \emph{respecting limit elements} --- and relate it to Kuratowski-Painlev\'e convergence. Finally, we examine related topologies, such as the Chabauty-Fell topology and the Hausdorff metric topology.
\begin{definition}
Let $\eta$ be a proper convergence on $\End(X)$ for some $X$. A net of closed sub-semigroups $(S_i)_{i \in I}$ of $\End(X)$ \emph{converges $\eta$-geometrically} to $S\subseteq \End(X)$ if and only if:
\begin{enumerate}[label=(\roman*)]
\item for each $s \in S$ there exist $s_i \in S_i$ such that $(s_i) \overset{\eta}{\to} s$,
\item for any subnet $(S_{j})$ of $(S_i)$, if $s_j \in S_j$ are such that $(s_j)\overset{\eta}{\to} s$ then $s\in S$.
\end{enumerate}
\end{definition}
\begin{definition}
Consider a net of actions $(\alpha_i)_{i \in I}$ on $X$ of a semigroup $T$, such that for every $t \in T$ the net $(\alpha_i(t))$ converges in $\eta$. Calling $S$ the set of limits of $(\alpha_i(t))$ for $t \in T$, we say the net $(\alpha_i)$ \emph{converges $\eta$-algebraically} to $S$. If $\eta$ is Hausdorff we can identify $S$ with the image of the map $\alpha:T \to \End(X)$ given by $(\alpha_i(t))\to\alpha(t)$.
\end{definition}
We extend the definitions naturally to monoids and groups. We leave verifying that these are proper convergences to the reader (from which it should become clear why geometric convergence is restricted to closed sub-semigroups). The first of the two points defining $\eta$-geometric convergence to a set $S$ is exactly what we call \emph{respecting $\eta$-limit elements}. It is clear that $\eta$-algebraic convergence also respects $\eta$-limit elements. Note that for both types of convergence, it it not clear whether the limit $S$ is itself a semigroup. If the convergence $\eta$ on $\End(X)$ is admissible, it can be easy shown that this is the case for both.
\begin{proposition}
\label{prop-admissible-limit-semigroup}
Let $\eta$ be an admissible convergence on $\End(X)$.
\begin{enumerate}[label=(\roman*)]
\item If $(S_i)_{i\in I}$ is a net of sub-semigroups of $\End(X)$ such that $(S_i)_{i\in I}\to S$ $\eta$-geometrically, then $S$ is a sub-semigroup of $\End(X)$.
\item If $\alpha_i:T\to \End(X)$ are actions converging $\eta$-algebraically to $S$, then $S$ is a sub-semigroup of $\End(X)$. If $\eta$ is Hausdorff, then the map $\alpha:T\to\End(X)$ given by $(\alpha_i(t))\to\alpha(t)$ is a homomorphism, and thus an action of $T$ on $X$.
\end{enumerate}
\end{proposition}
\begin{proof}
Considering any $s_1, s_2 \in S$, for either result we must show $s_1 s_2 \in S$. In the case of geometric convergence, we have $(s_{1,i}) \to s_1$ and $(s_{2,i}) \to s_2$ where the nets take values in the $S_i$. Letting $r_i = s_{1,i}s_{2,i}\in S_i$, by admissibility we have
$$(r_i(x)) = (s_{1,i} (s_{2,i}(x))) \to s_1(s_2(x)) = (s_1 s_2)(x)$$
for any $x \in X$. By the second point defining geometric convergence, $s_1s_2 \in S$.

With algebraic convergence we have some $t_1,t_2 \in T$ such that $(\alpha_i(t_1)) \to s_1$ and $(\alpha_i(t_2)) \to s_2$. We need to show that there is some $t \in T$ such that $(\alpha_i(t)) \to s_1s_2$ to show $s_1s_2 \in S$. By the admissibility of $\eta$, the natural candidate $t=t_1t_2$ works:
$$(\alpha_i(t_1t_2)) = (\alpha_i(t_1)\alpha_i(t_2)) \to s_1s_2.$$
If $\eta$ is Hausdorff we have $s_1s_2 = \alpha(t_1)\alpha(t_2)$ on one hand, and $(\alpha_i(t_1t_2))\to\alpha(t_1t_2)$ on the other. Thus $\alpha(t_1)\alpha(t_2) = \alpha(t_1t_2),$ and $\alpha$ is a homomorphism.
\end{proof}

Algebraic convergence allows for a natural characterization of approximation of sub-semigroups of $\End(X)$ in terms of their generators. Consider the free semigroup\footnote{This is the semigroup \gen{\N} with multiplication given by concatenation, whose elements are the finite-length ``words'' $n_1\cdots n_k$ of natural numbers $n_i \in \N$.} $F_\N$ on countable many generators, which has the universal property that any countably generated sub-semigroup $S$ of $\End(X)$ is the image of a homomorphism $\alpha: F_\N \to \End(X)$. If the convergence $\eta$ on $\End(X)$ is admissible then convergence of the generators is sufficient for algebraic convergence.
\begin{proposition}
\label{prop-admissible-generators}
Consider $\eta$ an admissible convergence on $\End(X)$, and a net of homomorphisms $\alpha_i: F_\N \to \End(X)$ for $i \in I$. The nets $(\alpha_i(n))$ $\eta$-converge for each $n \in \N$ if and only if $(\alpha_i)$ converges $\eta$-algebraically. If $\eta$ is Hausdorff the limit is the semigroup $\gen{\alpha(n)}_{n\in\N}$ where $\alpha$ is the limiting homomorphism.
\end{proposition}
\begin{proof}
The reverse implication is trivial. For the forward implication, consider any element $n_1\cdots n_k \in F_\N$, for some $k\in\N$; an induction on $k$, together with the admissibility of $\eta$, proves the result.
\end{proof}
This gives an easy way to construct maps between actions which are continuous in algebraic convergence. 
\begin{corollary}
\label{cor-algebraic-convergence-generators}
Suppose $(\alpha_i)$ converges $\eta$-algebraically. If $f: \End(X) \to \End(X)$ is continuous at the limits of $(\alpha_i(n))$ for each $n \in \N$ then the homomorphisms $n_1\cdots n_k \mapsto (f \circ \alpha_i(n_1)) \cdots (f \circ \alpha_i(n_k)) $ converge algebraically. If $\eta$ is Hausdorff the limit is the semigroup $\gen{f\circ\alpha_i(n)}_{n\in N}.$
\end{corollary}
The results above help one understand in what sense one can learn a semigroup by learning a set of generators.  While we do not do so here, one should be able to use a similar construction to describe learning a Lie group in terms of both discrete generators, as above, and continuous generators (Lie algebra elements). Absent this, geometric convergence may be a more natural description of approximating the continuous by the discrete. For example in \cite{Yarotsky2018} the discrete additive groups $\frac{1}{n}\Z^d$ are used to approximate $\R^d$ as $n\to \infty$ (in order to prove certain neural networks approximate $\R^d$-equivariant functions), and this approximation is best understood geometrically.\footnote{One might then ask: for a semigroup $T$ and a net of its actions $(\alpha_i)$ on $X$, what is the relation between the $\eta$-geometric convergence $(\alpha_i(T)) \to S$ and the $\eta$-algebraic convergence $(\alpha_i)\to \alpha(T)$ (supposing $\eta$ is Hausdorff)? The best we have is that if both limits exist then $\alpha(T) \subseteq S$ \cite[Chapter 9]{Thurston2022}.}

Having gained some understanding of two natural notions of convergence, we now relate their shared property of respecting limit elements to known topologies. For a net of subsets $(A_i)$ of a convergence space $Y$ we define the \emph{lower limit}
$$\text{Li}(A_i) = \{y \in Y: \exists y_i \in A_i,~ (y_i) \to y\}.$$
This definition is standard for topological spaces $Y$ \parencite{Beer1993,Beer2010}, and has been formulated for convergence spaces in terms of filters in \cite{Dolecki2009}. (The above is also sometimes known as the \emph{inner limit} or even \emph{limit inferior}, although the latter term risks confusion with set theoretic limits.) A net $(A_i)$ is called \emph{lower Kuratowski-Painlev\'{e} convergent} to $A$ if $A \subseteq \text{Li}(A_i)$. Note that by definition this convergence respects limit elements, and furthermore \emph{any} convergence respecting limit elements is stronger than lower Kuratowski-Painlev\'{e} convergence. We could in fact give this as the definition of a convergence respecting limit elements. There is also a notion of \emph{upper limit}
$$\text{Ls}(A_i) = \{y \in Y: \exists \text{ a subnet } (A_j) , ~y_j \in A_j,~ (y_j) \to y\},$$
 and a corresponding \emph{upper Kuratowski-Painlev\'{e} convergence} to $A$, if $\text{Ls}(A_i) \subseteq A$. Plain \emph{Kuratowski-Painlev\'{e} convergence} of $(A_i)$ to $A$ is defined to hold when $\text{Ls}(A_i) = A = \text{Li}(A_i)$. Note that this coincides with geometric convergence, which is thus Hausdorff.

While Kuratowski-Painlev\'{e} convergence is not generally topological, it is known that if $Y$ is a topological space then \emph{lower} Kuratowski-Painlev\'{e} convergence corresponds to a topology, and the limiting sets are closed in $Y$. We thus discuss certain natural topologizations of the space of closed subsets of $Y$ (known as \emph{hypertopologies}), which we denote by $2^Y$ --- in particular those stronger than the lower Kuratowski-Painlev\'{e} topology. It is helpful to keep in mind the main case of interest, $Y\subseteq\End(X)$ with an admissible topology, for some $X$.

The topology of lower Kuratowski-Painlev\'{e} convergence is also known as the \emph{lower Vietoris topology}, usually defined by the subbase consisting of $2^Y$ together with the sets
$$U^- = \{ F \in 2^Y : U\cap F \ne \emptyset\}$$
where the $U \subseteq Y$ are open. There is a standard family of topologies stronger than the lower Vietoris topology --- meaning exactly that convergence respects limit elements --- the so-called \emph{hit-and-miss} topologies. These are topologies which are the supremum of the lower Vietoris topology and another topology, with a base of sets of the form
$$V^+ = \{ F \in 2^Y : F\subseteq V\}$$
for some specified family of sets $V \subseteq Y$. If the sets $V$ vary over the open sets we obtain the \emph{Vietoris topology} (or, following \textcite{Michael1951}, the \emph{finite} or \emph{F-topology}). If the sets $V$ vary over the complements of compact sets, we obtain the \emph{Chabauty-Fell topology}, equivalent to the Vietoris topology if $Y$ is compact.

We note that when $(Y,d)$ is a compact metric space, the Vietoris topology agrees with the topology induced by the Hausdorff metric
$$d(A,B) = \max\{ \sup_{y \in Y} d(y,A), \sup_{y \in Y} d(y,B) \},$$
where we define 
$$d(y,A) = \inf_{a \in A} d(y,a).$$
When $Y$ is a locally compact space, it is fruitful to pass to its one-point compactification and use the Vietoris topology, or equivalently the Chabauty-Fell topology. If $Y$ is also a metric space this again coincides with the Hausdorff metric topology by appropriately extending the metric. Bringing us full circle, if $Y$ is additionally a (second countable) semigroup, and we restrict the topology from $2^Y$ to the space of closed sub-semigroups, we obtain the topology of geometric convergence (equivalently, Kuratowski-Painlev\'e convergence) \parencite{Baik2012, delaHarpe2008}.

The final moral is that to obtain nice properties of the Chabauty-Fell topology on $2^Y$ we ought to restrict ourselves to the case that $Y$ is locally compact and Hausdorff. If $Y$ is Hausdorff, it being locally compact is equivalent to $2^Y$ being Hausdorff: local compactness of a Hausdorff $Y$ in fact guarantees $2^Y$ is locally compact Hausdorff \parencite[Proposition 5.1.2]{Beer1993}. Furthermore, Kuratowski-Painlev\'e convergence of nets coincides with Chabauty-Fell convergence if the space $Y$ is locally compact Hausdorff \parencite[Theorem 5.2.6]{Beer1993}. For the convergences to coincide on \emph{sequences}, however, it suffices that $Y$ be first countable\footnote{We do not review the definition of first countability here, it sufficing to note second countable spaces are first countable.} Hausdorff \parencite[Theorem 5.2.10]{Beer1993}.

\section{The learnability tradeoff}
\label{sec-main-result}
We present our main result in this section. We begin with the abstract statement of the main theorem and its proof, before interpreting it in the context of learning symmetries. We end with a discussion of the hypotheses and some cases in which they are satisfied. 

For the statement, we introduce some notation relating to binary relations. Recall that a \emph{binary relation} on sets $X$ and $Y$ is a set $R \subseteq X\times Y$. We write $X_R$ to indicate those elements $x \in X$ which occur in some pair $(x,y) \in R$, and define $Y_R$ similarly. 
\begin{theorem}
\label{thm-main-thm}
Let $F,\Gamma$ be convergence spaces and $R \subseteq F\times \Gamma$ a relation closed in the product convergence, and fix a subset $A\subseteq R$. Suppose $H:\Gamma\to \Gamma$ has $(f,H(\gamma)) \in R$ for all $(f,\gamma) \in A$. If $H$ is continuous at $\gamma\in \Gamma$, then 
\begin{align}
\label{eq-subset}
(f,\gamma) \in \overline{A} \Rightarrow (f,H(\gamma)) \in R.
\end{align}
Suppose $(f,H(\gamma)) \in A$ (rather than just $R$) for all $(f,\gamma) \in A$, and $(f,H(\gamma)) \in R$ implies $(f,\gamma) \in R$. If $H$ is continuous at $\gamma$, and $(f,\gamma) \in \overline{A}$ for any $f\in F_{\overline{A}}$ such that $(f, \gamma) \in R$, then
\begin{align}
\label{eq-equal}
(f,\gamma) \in \overline{A} \Leftrightarrow (f,H(\gamma)) \in \overline{A}.
\end{align}
\end{theorem}
\begin{proof}
Suppose $H$ and $\gamma$ satisfy the conditions for (\ref{eq-subset}), and $(f,\gamma) \in \overline{A}$. Consider a net $((f_i, \gamma_i))_{i \in I}$ with values in $A$ such that $((f_i,\gamma_i)) \to (f, \gamma)$. By continuity $((f_i, H(\gamma_i))) \to (f,H(\gamma))$. By assumption $(f_i, H(\gamma_i)) \in R$ for each $i \in I$, so $(f, H(\gamma)) \in \overline{R} = R$.

Suppose now $H$ and $\gamma$ meet the conditions for (\ref{eq-equal}). If $(f,\gamma) \in \overline{A}$ a similar continuity argument shows $(f, H(\gamma)) \in \overline{A}$. Suppose on the other hand $(f, H(\gamma)) \in \overline{A}$. Note that $\overline{A}\subseteq \overline{R}=R$, so $(f,\gamma) \in R$. Since we also have $f \in F_{\overline{A}}$, it follows that $(f,\gamma) \in \overline{A}$.
\end{proof}
The proof is elementary, and the power of the result comes from an appropriate instantiation of the hypotheses. We interpret $F$ as a set of ``objects,'' and $\Gamma$ as a set of ``symmetries'' (each symmetry being for example a group of transformations, rather than an individual transformation), and take $R$ to be those pairs $(f,\gamma)$ such that ``$f$ has symmetry $\gamma$.'' We then call any $(f,\gamma)\in R$ a \emph{symmetric pair}, or say that $f$ is \emph{symmetric under $\gamma$}. The choice of convergences on $F$ and $\Gamma$ corresponds to saying what it means to ``learn'' an object or a symmetry by approximation. As shown in examples below, this approach is flexible enough to describe most senses of approximation of interest in machine learning. Finally, we call $A$ the \emph{symmetric ansatz}, thinking of it as those symmetric pairs $(f,\gamma)$ one can use to approximate other pairs. We therefore call pairs $(f,\gamma)\in\overline{A}$ \emph{learnable}, extending the term to apply to any $f\in F_{\overline{A}}$ and $\gamma\in \Gamma_{\overline{A}}$. We call a map $H:\Gamma \to \Gamma$ satisfying either set of hypotheses a \emph{symmetry non-uniqueness} for the ansatz, such an $H$ generally demonstrating some $f$ has multiple symmetries. We often use the term to mean an $H$ which is ``non-trivial.''

Temporarily leaving aside the remaining topological hypotheses, we are now ready to give interpretations to the conclusions (\ref{eq-subset}) and (\ref{eq-equal}). The former says that if any symmetric pair $(f,\gamma)$ in the ansatz transforms to a symmetric pair $(f,H(\gamma))$, the same is true for any learnable pair. The following is an immediate consequence.
\begin{corollary}
\label{cor-unlearnable}
If $(f,\gamma) \in R$ and $H:\Gamma\to \Gamma$ satisfy the hypotheses of (\ref{eq-subset}) such that $(f,H(\gamma))$ is not symmetric, then $(f,\gamma)$ is not learnable.
\end{corollary}
Conclusion (\ref{eq-equal}) is even stronger, saying essentially that one cannot distinguish between $\gamma$ and $H(\gamma)$ based on which learnable objects $f$ have them as symmetries; this, provided a certain condition on $\gamma$ holds, $H$ maps pairs in the ansatz to the ansatz, and $H$ only maps symmetric pairs to symmetric pairs. In the case of groups, the latter condition on $H$ is something like saying $H(\gamma)$ contains $\gamma$ as a subgroup.  The condition on $\gamma$, that $(f,\gamma) \in \overline{A}$ for any $f\in F_{\overline{A}}$ with $\gamma$ as a symmetry, can be seen as a kind of local universal approximation statement: fixing $\gamma$, for any learnable object $f$, if $(f,\gamma)$ is symmetric then it is learnable (approximable with the symmetric ansatz). We obtain the following tradeoff between learnability of pairs and the ability to identify a symmetry.
\begin{corollary}[Learnability tradeoff]
\label{cor-learnability-tradeoff}
Suppose $H:\Gamma\to \Gamma$ is continuous and satisfies the hypotheses of (\ref{eq-equal}). Then for any $\gamma\in \Gamma$, the following cannot both be true:
\begin{enumerate}[label=(\roman*)]
\item if $f$ is learnable and symmetric under $\gamma$ then $(f,\gamma)$ is learnable,
\item there exists a learnable $f$ symmetric under one of $\gamma$ and $H(\gamma)$ but not the other.
\end{enumerate}
\end{corollary}

Note that $f$ being learnable can be quite a weak statement. If the ansatz contains a trivial symmetry $\gamma_0 \in \Gamma$, such that $(f, \gamma_0) \in R$ for all $f\in F$, then $F_{\overline{A}} = \overline{F_A}$. If additionally the ansatz contains a universal approximating family for $F$, meaning $F = \overline{F_A}$, then $F_{\overline{A}}= F$.

We return to the topological issues of $H$ being continuous and $R$ closed. The continuity of $H$ is something we leave to be checked on a case-by-case basis.\footnote{The existence of such a function may in certain cases be provable by a selection theorem; see for example the classic work of \textcite{Michael1951}.} While the restriction that $R$ be closed is required for the proof of our result, it is not unnatural (and might in fact be desirable in machine learning applications). Intuitively, closedness says that if a pair $(f,\gamma)$ is such that $f$ and $\gamma$ can simultaneously be approximated by symmetric pairs, then $(f,\gamma)$ itself is symmetric. We study how closedness can arise in the case that $F$ is a space of functions and $\Gamma$ a space of coupled actions, and $R$ the equivariant pairs thereof.

The case of continuous functions $F \subseteq C(X,Y)$ between fixed convergence spaces $X$ and $Y$ is prevalent, and sufficiently general for our purposes. Letting $X$ and $Y$ be fixed has the advantage that coupled semigroup actions $\gamma:S\to \End(X)\times \End(Y)$ can each be identified with a sub-semigroup of $\End(X)\times \End(Y)$. To place a convergence on $\Gamma$ is to choose a convergence on (a subset of) the set of sub-semigroups of $\End(X)\times \End(Y)$. The following fact may guide this choice.
\begin{proposition}
\label{prop-equivariance-closed}
Let $X$ and $Y$ be convergence spaces, with $Y$ Hausdorff. Give $F$, $\End(X)$, and $\End(Y)$ convergences, with those on $F$ and $\End(Y)$ being admissible, and give $\End(X)\times \End(Y)$ the product convergence. Give $\Gamma$ a convergence that respects limit elements of the said product convergence. Then the relation $R$ consisting of equivariant pairs $(f,\gamma)$ is closed in the product convergence on $F\times \Gamma$.
\end{proposition}
\begin{proof}
Consider a pair $(f,\gamma) \in \overline{R}$. There exists a net $(f_i, \gamma_i)$ of elements of $R$ which converges to $(f,\gamma)$ in the product of the convergences on $F$ and $\Gamma$. Consider any $x \in X$ and $s \in \gamma$. By the assumption of respecting limit elements, there exist $s_i=({s_i}_X, {s_i}_Y) \in \gamma_i$ such that $(s_i) \to s$ in the product convergence on $\End(X)\times \End(Y)$. Now note that on one hand, the assumed admissibility of the convergence on $F$ means
$$ f_i({s_i}_X (x)) \to f(s_X(x)).$$
On the other hand, the admissibility of $\End(Y)$ gives
$$ {s_i}_Y(f_i(x)) \to s_Y(f(x)).$$
But the nets $(f_i({s_i}_X (x)))_{i \in I}$ and $({s_i}_Y(f_i(x)))_{i \in I}$ are term-wise equal, since $(f_i, \gamma_i) \in R$. Being nets in $Y$, which is Hausdorff, the limits $f(s_X(x))$ and $s_Y(f(x))$ are equal. As this holds for any $x \in X$ and $s \in \gamma$, we have that $(f,\gamma)$ is an equivariant pair and therefore in $R$.
\end{proof}

We thus see the relevance of Sections \ref{sec-admissibility} and \ref{sec-hypertopology}. We end this section with some examples in which $R$ is closed, leaning on these previous discussions.
\begin{example}
\label{ex-setups}
Most of these examples are in the case of equivariant functions, applying the above result. The final example, however, shows how the results of this section may also apply in an unsupervised learning setting.
\item
\begin{enumerate}[label=(\roman*)]
\item A common class of examples arises where $X$ and $Y$ are locally compact Hausdorff topological spaces. We can let $F = C(X,Y)$ have the compact-open topology, which is admissible. Similarly, $\End(X)\times \End(Y)=C(X)\times C(Y)$ (if no extra structure is assumed) may be given the product of the compact-open topologies. Finally, either of the two natural convergences on $\Gamma$ discussed in Section \ref{sec-hypertopology}, geometric or algebraic, respects limit elements.

\item A special case of the above ubiquitous in machine learning is the case of vector spaces $X = \R^m$ and $Y = \R^n$. The compact-open topology on $F=C(X,Y)$ is then the topology of uniform convergence on compact sets --- the standard topology used universal approximation theorems. The space $\End(X) \times\End(Y)$ is a product of matrix spaces, which are Euclidian and thus locally compact. Thus $\Gamma$ can be taken as the space of all closed sub-semigroups of such pairs of matrices acting on $X \times Y$, and given the Chabauty-Fell topology. If one attempts to learn semigroups by learning generators, one instead gives $\Gamma$ the algebraic convergence of actions of the appropriate free group.

\item \label{ex-setups-discrete} This next example is quite useful, in being easy to apply. It demonstrates the learnability tradeoff between universal approximation of equivariant functions for some fixed symmetry, and the ability to ``point-identify'' the symmetry --- that is, learn it exactly, rather than through approximation. Point-identification corresponds to convergence in a discrete topology, for which all functions are continuous; thus we need not check the continuity of the function $H$ transforming between symmetries when applying Corollary \ref{cor-learnability-tradeoff}.

We consider $F$ a collection of functions from a set $X$ to a Hausdorff convergence space $Y$. Giving $X$ the convergence of the discrete topology, pointwise convergence on $F$ is admissible: $(x_j)_{j \in J}\to x$ exactly when $x_j = x$ eventually (for all $j \in J$ large enough). For any net $(f_i)$ in $F$ we thus have $(f_i(x_j)) = (f_i(x))$ eventually, and $(f_i(x)) \to f(x)$ by definition if $(f_i) \to f$. The discrete topologies on $\End(X)$ and $\End(Y)$ are of course admissible (as they are for any $X$ and $Y$), and geometric convergence on the space of sub-semigroups of $\End(X)\times\End(Y)$ then itself corresponds to the discrete topology on that space --- that is, point-identification of the symmetry.

\item We finally show how the learnability tradeoff may arise in unsupervised learning. Considering a probability space $(\Omega, \Sigma, \PP)$ and a Polish metric space $S$, we let $F$ be a set of random variables taking values in $S$, and $\Gamma$ be a set of semigroups acting on the set $S$ (not necessarily preserving any metric structure). We make $F$ a convergence space by endowing it with convergence in distribution, and give $\Gamma$ a convergence that respects limit elements of uniform convergence on compact sets (for example the geometric convergence with respect to uniform convergence on compact sets). We claim that the following relation $R$ is closed: pairs $(X, T)\in F \times \Gamma$ such that the distribution of $X$ is $T$-invariant, that is, $X$ and $t\cdot X$ have the same distribution for all $t \in T$.

The main tools in the proof are the admissibility of uniform convergence on compact sets with respect to almost-sure convergence and the Skorokhod representation theorem. For the admissibility statement, note that if $X_n \asto X$ and $t_n \to t$ as functions on $S$ uniformly on compact sets, then the admissibility of the latter convergence with respect to convergence in the metric on $S$ carries over to almost-sure convergence; explicitly, for almost all $\omega \in \Omega$ we have that $X_n(\omega) \to X(\omega)$ and thus $t_n \cdot X_n(\omega) \to t\cdot X(\omega)$.

Suppose now that $X_n \to X$ in distribution and $T_n \to T$ in $\Gamma$, where the distribution of each $X_n$ is $T_n$-invariant. To show $R$ is closed we must show $X$ has $T$-invariant distribution. For any $t \in T$ we have $t_n \in T_n$ such that $t_n \to t$ uniformly on compact sets. Next, the Skorokhod representation theorem says that there exist random variables $Y_n$ and $Y$ with the same respective distributions as $X_n$ and $X$, such that $Y_n \asto Y$. By the admissibility shown above, $t_n \cdot Y_n \asto t \cdot Y$. Thus $t_n \cdot X_n \to t \cdot X$ in distribution. But $t_n \cdot X_n$  and $X_n$ have the same distribution for each $n$, so the limits $t \cdot X$ and $X$ are equal in distribution. This holds for arbitrary $t \in T$, so the result is proved.\qedhere
\end{enumerate}
\end{example}

\section{Symmetry non-uniqueness in neural networks}
\label{sec-ansatzes}
The purpose of this section is to study symmetry non-uniqueness in concrete equivariant neural network architectures. Recall from Section \ref{sec-intro-motivation} that a common design pattern exists. Generally, equivariance is guaranteed by ensuring each layer is equivariant. The specific pattern often used to this end is as follows. First, interest is expressed in some class of groups (or semigroups) acting between layers; this could be for example all groups acting linearly, or if the outputs are interpreted as ``signals'' over a space, those groups which act on the signals by acting on the underlying space. To make each layer equivariant under a particular group from this class, the linear map is made equivariant under the particular group, while the non-linearities are chosen to be equivariant under \emph{any} group from the class (usually under the same action on the input and output).

The pattern described above leaves room for non-uniqueness. In particular, one may map between groups of the chosen class in a way that preserves equivariance of the linear maps. Since the non-linearities remain equivariant, this means the entire network remains equivariant under the map.

We describe in Section \ref{sec-lin-eq} such a family of networks we call \emph{linearly equivariant}, which constitute a plausible first attempt for learning general symmetries: roughly, having symmetry completely determined by the linear layers. This is a relevant idealization to study. Equivariant networks without non-linearities have been studied by \textcite{Lawrence2022}, and appear in some experiments of \textcite{Dehmamy2021}. Additionally, networks which are in simple cases essentially linearly equivariant arise in the work of \textcite{EMLP} (see Section \ref{sec-emlp}). Such networks, however, have a symmetry non-uniqueness arising in the manner described above, corresponding with taking the algebra of a semigroup.
\begin{remark}
We find that linearly equivariant networks, while possible to construct (at the very least by having trivial, or equivalently, no non-linearities, as in \cite{Lawrence2022}), are generally difficult to formulate. Indeed if one uses the design pattern described above, the only linearly equivariant networks are linear functions: a ``non-linearity'' which commutes with all linear maps is necessarily a scalar multiple of the identity.\footnote{If $\sigma:\R^n\to\R^n$ has $\sigma(Ax)=A\sigma(x)$ for all $A \in \End(\R^n)$, then $\sigma(x)=A_x \sigma(e_1)$ where $A_x$ is any matrix taking the first basis vector $e_1$ to $x$. Since $(A_x+A_y)e_1 = x+y$ we obtain that $\sigma(x+y) = (A_x+A_y)\sigma(e_1) = \sigma(x)+\sigma(y)$. That is, $\sigma$ is linear. The only linear maps commuting with all other linear maps are scalings of the identity.} This difficulty is reminiscent of results of \textcite{Diffeo2022} demonstrating the impossibility of using the design pattern to build non-trivial neural networks for vector fields equivariant under diffeomorphisms: non-linearities which commute with any such symmetry are necessarily trivial.
\end{remark}

In Section \ref{sec-gcnns} we study group\hyphen{}convolutional neural networks and their generalizations. We are able to generalize in terms of kernel operators the well-known theorem of \textcite{KondorTrivedi2018} which states that under certain conditions the only linear equivariant maps are given by group convolution. Unlike the algebraically involved proof in \cite{KondorTrivedi2018}, our proof is elementary, although passing from kernels operators to group convolutions requires some machinery. We then show in Theorem \ref{thm-gcnn-uniqueness} that in many cases, group convolutions have the appealing property of not having symmetry non-uniqueness if groups are restricted to be invariances of a certain measure; this result applies for example to the work of \textcite{Zhou2021}. On the other hand, we show in Section \ref{sec-semigroup-convolution} that semigroup convolutions generally have large symmetry non-uniquenesses, although ``group-like'' semigroup convolutions over homogeneous spaces, such as those used by \textcite{Dehmamy2021}, may not.

\subsection{Linearly equivariant neural networks}
\label{sec-lin-eq}
Fix (not necessarily finite-dimensional) vector spaces $V_1,\ldots,V_n$ over a field $K$. We consider here neural networks of the form
$$f: V_1 \overset{L_1}{\to} V_2 \overset{\sigma_1}{\to} V_2 \overset{L_2}{\to} V_3 \cdots \overset{L_{n-1}}{\to} V_n \overset{\sigma_{n-1}}{\to} V_{n}$$
where $L_k: V_k \to V_{k+1}$ are linear maps. One intuition is that learning symmetries becomes simpler if equivariance is controlled only by the linear maps $L_k$. We formalize this notion as follows. Note that in this section we do not require actions on vector spaces to be linear, so $\End(V_k)$ can always be taken to mean $\End(\set(V_k))$.
\begin{definition}
Let $\Gamma$ be a set of coupled actions on the vector spaces (or their underlying sets). A neural network $f$ is $\Gamma$-\emph{linearly equivariant} if for any collection $(\alpha_k=(\alpha_{V_k},\alpha_{V_{k+1}}))_{k=1}^{n-1}$ of coupled actions in $\Gamma$ of the same semigroup $S$,
$$ \alpha_{V_{k+1}}(s) \circ L_k = L_k \circ \alpha_{V_k}(s) ~~~\forall k \in [n-1],~s\in S ~~~\Rightarrow~~~  \alpha_{V_n}(s) \circ f = f \circ  \alpha_{V_1}(s)~~\forall s\in S.$$
When $f$ satisfies both sides of the implication above for a particular choice of $(\alpha)=(\alpha_k)_{k=1}^{n-1}$ we say $f$ is $(\alpha)$-linearly equivariant.
\end{definition}
If $f$ is $\Gamma$-linearly equivariant where $\Gamma$ contains \emph{all} coupled actions on the vectors spaces $V_k$ of $f$, we simply call $f$ \emph{linearly equivariant}. Intuitively, such networks are maximally flexible among those that are $\Gamma$-linearly equivariant for some $\Gamma$, in terms of being able to represent and potentially learn any symmetry which involves the given vector spaces.\footnote{In practice, one would require a way to simultaneously parameterize the actions $\alpha_{k}$ for all $k \in [n]$, to make these learnable. Then one could parameterize the linear maps $L^{(\alpha,\theta)}_k:V_k \to V_{k+1}$ by a coupled action $\alpha$ and additional parameter $\theta$, such that that $L^{(\alpha,\theta)}_k$ is guaranteed to be equivariant with respect to $\alpha$ for any choice of $\theta$.} The flexibility of such a network essentially guarantees that it will be equivariant under multiple actions, in which case the results of the previous section apply. The reasoning begins with the following observation.
\begin{lemma}
\label{lem-lin-eq}
A neural network $f$ is linearly equivariant if and only if the following holds: for any maps $t_k \in \End(V_k)$, if $t_{k+1}\circ L_k = L_k \circ t_k$ for all $k\in[n-1]$, then $t_n \circ f = f \circ t_1$.
\end{lemma}
\begin{proof}
The condition implies linear equivariance by definition. Suppose then $f$ is linearly equivariant, and $t_k \in \End(V_k)$ are such that $t_{k+1}\circ L_k = L_k \circ t_k$ for all $k\in[n-1]$. Let $S=\gen{t}$ be the semigroup generated by $t=(t_1,\ldots, t_n)$, where multiplication is defined in the obvious component-wise way. For every $k \in [n-1]$, define the coupled action $\alpha_k=(\alpha_{V_k},\alpha_{V_{k+1}})$, where $\alpha_{V_k}$ takes the $k$-th component of any element of $S$. We have $\alpha_{V_{k+1}}(t)\circ L_k = L_k \circ \alpha_{V_k}(t)$ by assumption.  A standard inductive argument then shows that $\alpha_{V_{k+1}}(s)\circ L_k = L_k \circ \alpha_{V_k}(s)$ for all $s \in S$. By linear equivariance, $\alpha_{V_n}(t) \circ f = f \circ \alpha_{V_1}(t)$, that is, $t_n \circ f = f \circ t_1$.
\end{proof}
The lemma extends to linear equivariance defined for actions monoids or actions of groups, with minimal additions to the proof.

Linear equivariance of a network means equivariance ``spills over'' from a semigroup to the semigroup algebra. The next result follows from Proposition \ref{prop-linear-equiv-algebra} and Lemma \ref{lem-lin-eq}.
\begin{proposition}
\label{lin-eq-alg}
Consider a collection of coupled actions $(\alpha)=(\alpha_k)_{k=1}^{n-1}$ of a semigroup $S$ on the $V_k$. A linearly equivariant neural network $f$ is $(\alpha)$-linearly equivariant if and only if it is $(\tilde{\alpha})$-linearly equivariant, where $(\tilde{\alpha})=(\tilde{\alpha}_k)_{k=1}^{n-1}$ with each $\tilde{\alpha}_k$ the action of the semigroup algebra $K[S]$ induced by $\alpha_k$.
\end{proposition}
Again, the result extends easily to monoids and groups. For a group $G$, however, it is useful to consider the result in terms of $K[G]^\times$, the invertible elements of the group algebra.

Theorem \ref{thm-main-thm} and its corollaries thus apply to the following situation. Consider $F$ a set of functions between two vector spaces $U$ and $V$, and $\Gamma$ a set of coupled actions. Fix convergences on $F$ and $\Gamma$, such that the equivariance relation $R \subseteq F\times \Gamma$ is closed under the product convergence, and the map $H:\alpha\mapsto\tilde{\alpha}$ is continuous. Finally let the symmetric ansatz $A$ be pairs $(f,\alpha)$ where the function $f$ representable as a linearly equivariant neural network, which is $(\alpha)$-linearly equivariant for $(\alpha)$ a collection of coupled actions on the layers, such that the actions on the first and last layer form the coupled action $\alpha$. This can be made more concrete with the following simple example. 

\begin{example}
Let $U = \R^2$ and $V=\R$. Let $F$ consist of all continuous functions $U\to V$, and $\Gamma$ be coupled representations on $U,V$. Suppose $A$ consists of pairs $(f,\rho)$ with linearly equivariant networks $f:\R^2 \to \R$ that are $(\rho)$-linearly equivariant. Consider the function $f(x,y) = x^2 + y^2$. This function is equivariant under the coupled representation $\rho$ consisting of the canonical representation of the permutation group $\rho_{\R^2}:S_2\to\GL(\R^2)$ and the trivial representation $\rho_{\id} : S_2 \to \{\id\}$. We show it is not equivariant under the corresponding coupled representation $\tilde{\rho}$ of $\R[S_2]^\times$. Consider the linear combination $a = \id + \sqrt{2}(1~2)$ of the identity in $S_2$ with the transposition $(1~2)$. (This has inverse $a\inv = -\id + \sqrt{2}(1~2)$, so $a$ is indeed in $\R[S_2]^\times$.) Observe that
$$f(\tilde{\rho}_{\R^2}(a) (x,y)) = f\left(x+\sqrt{2}y, \sqrt{2}x+y\right) = 3(x^2+y^2) + 4\sqrt{2}xy$$
is not equal to
$$\tilde{\rho}_{\id}(a) f(x,y) = \left(1 + \sqrt{2}\right) (x^2+y^2).$$
So $(f,H(\rho)) \not\in R$, where $H(\rho) = \tilde{\rho}$. From this we can deduce several facts using Theorem \ref{thm-main-thm}. First, we consider the combination of the topology of pointwise convergence on $F$ and the discrete topology on $\Gamma$, discussed in Example \ref{ex-setups} \ref{ex-setups-discrete}. Since $H$ is of course continuous in the discrete topology, $(f,\rho)$ is not in $\overline{A}$. In particular $f$ is not even pointwise learnable by $(\rho)$-linearly equivariant networks.

An important feature of this argument is that the equivariant ansatz does not need to contain, or even be able to approximate, the representation $\tilde{\rho}$. The ansatz could even only contain representations of finite groups. The group algebra is only needed to witness a transformation of representations under which linearly equivariant neural networks remain equivariant but some function in $F$ does not. If the ansatz does contain the representation $\tilde{\rho}$, then Theorem \ref{thm-main-thm} implies that not only do there exist non-learnable functions in $F$, but $\rho$ and $\tilde{\rho}$ cannot be distinguished by the learnable ones.

Furthermore, if we consider some of the non-trivial convergences one may place on $F$ and $\Gamma$ which satisfy Proposition \ref{prop-equivariance-closed}, Theorem \ref{thm-main-thm} says we cannot even approximate $f$ and $\rho$ simultaneously with linearly equivariant networks, if $H$ is continuous. For geometric convergence, the question of whether $H$ is continuous is not quite simple.\footnote{We believe it is not too hard to show that it is continuous on finite semigroups.} If the ansatz is restricted to semigroups with finitely many generators, say one generator $M \in \End(\R^2)$, then $H$ itself cannot quite be continuous with respect to algebraic convergence, since the semigroup algebra will not be finitely generated. By Corollary \ref{cor-algebraic-convergence-generators}, however, any map taking the generator to the algebra of its semigroup, say $M \mapsto M+\id$, induces a continuous map $H$ between the generated semigroups: a symmetry non-uniqueness. Theorem \ref{thm-main-thm} then says $f$ and $\rho$ cannot be simultaneously approximated by linearly equivariant networks.
\end{example}

\subsection{Equivariant multilayer perceptrons}
\label{sec-emlp}
We turn to the ``equivariant multilayer perceptrons'' of \textcite{EMLP}. With a certain simplification, these provide a valuable example of a family of neural networks which have non-trivial non-linearities, but are essentially linearly equivariant.

Equivariance of the linear layers is ensured by parametrizing them to be equivariant under the generators of a group (both in the discrete and Lie algebra sense). By construction, vector spaces in the layers are decomposed into certain direct sums. Given any group $G$, equivariance of each layer mapping $V_k=\bigoplus_{r=1}^{ m_k} V^{(r)}_{k}$ to $V_{k+1}$ is ensured as follows. First, the layer contains a pair maps, $L_k:V_k\to V_{k+1}$ linear and equivariant, and $s_k=\bigoplus_{r=1}^{ m_{k+1}} s^{(r)}_{k}$ for $s^{(r)}_k : V_{k+1}^{(r)} \to \R$ all linear and invariant. Then a bilinear map $B_k:V_{k+1}\to V_{k+1}$ (for some decomposition of $V_{k+1}$ into two components) is applied to the output of $L_k$, and finally the functionals $s^{(r)}_k$ are used as a ``gated non-linearity'' scaling each $r$-th component of the output of $B_k\circ L_k$. That is, the $k$-th layer maps
$$ x \mapsto \bigoplus_{r=1}^{ m_{k+1}} \sigma(s^{(r)}_kx) (B_k(L_kx))^{(r)}. $$
In \cite{EMLP}, parity arguments are used to show that for certain groups, certain equivariant functions are not learnable
using layers without bilinear maps,
$$ x \mapsto \bigoplus_{r=1}^{ m_{k+1}} \sigma(s^{(r)}_kx) (L_kx)^{(r)}. $$
The non-universality of such networks can be proven more broadly, by noting they are essentially linearly equivariant. Namely we have the following version of Lemma \ref{lem-lin-eq} for a network $f:V_1\to V_n$ of the form above: for any transformations $t_k \in \GL(V_k)$ such that $t_{k+1}L_k =L_k t_k$ and $s_k= s_k t_k$ for all layers $k\in[n-1]$, we have $t_n \circ f = f \circ t_1$. This is because with the given hypotheses, the $k$-th layer maps $t_k x$ to
\begin{align*}
\bigoplus_{r=1}^{ m_{k+1}} \sigma(s^{(r)}_kt_kx) (L_k t_k x)^{(r)} = \bigoplus_{r=1}^{ m_{k+1}} \sigma(s^{(r)}_kx) (t_{k+1}L_k x)^{(r)} = t_{k+1}\bigoplus_{r=1}^{ m_{k+1}} \sigma(s^{(r)}_kx) (L_k x)^{(r)},
\end{align*}
that is, $t_{k+1}$ applied to the output of the layer given $x$. The composition of these equivariant layers satisfies $t_n \circ f = f \circ t_1$. The results of Section \ref{sec-lin-eq} thus apply if one restricts from $K[G]$ to the sub-semigroup\footnote{$K[G]_1$ is indeed closed under multiplication, noting that $\beta: a\mapsto \sum_{g\in G} a(g)$ is a representation of $K[G]$ on scalars, so $\beta(ab) = \beta(a)\beta(b)=1$ for any $a,b \in K[G]_1$.} $K[G]_1$ of elements $a\in K[G]$ for which $\sum_{g\in G} a(g) = 1$. This is because if $s_k\alpha(g) = s_k$ for all $g \in G$ then $s_k \tilde{\alpha}(a) = s_k$ for all $a\in K[G]_1$.  
\begin{remark}
The theory of Section \ref{sec-main-result} was developed in part to explain unsuccessful unpublished experiments conducted by the author, attempting learn groups using equivariant multilayer perceptrons modified to have learnable group generators. The experimental results agree with theory, the learned generators belonging to $K[G]_1^\times$ rather than $G$. Their belonging to $K[G]_1^\times$ rather than some bigger set is the subject of Section \ref{sec-max}.
\end{remark}

\subsection{Group-convolutional neural networks}
\label{sec-gcnns}
Let $X$ and $Y$ be measurable spaces with measurable actions of $G$ --- which for brevity we denote in this section by left-multiplication, as in $x\mapsto gx$. Group-convolutional neural networks are used in the context of ``signals'' defined on spaces with actions of $G$, such as functions $f:X\to \R$. Equivariance of maps between signals is sought under a natural action defined on signals, given by $g\cdot f = f \circ g\inv$.

We suppose for the remainder of this section that $\mu$ and $\nu$ are $G$-invariant measures on $X$ and $Y$ respectively: $\mu = \mu \circ g\inv$ for all $g \in G$, and similarly for $\nu$. All measures in this section are assumed to be $\sigma$-finite, and topological spaces second countable.

We first give a measure-theoretic version of the theorem in \cite{KondorTrivedi2018} characterizing equivariant linear layers. Next, we use this analysis to show symmetry non-uniqueness in group\hyphen{}convolutional neural networks can be avoided by restricting to the invariances of a measure.

\subsubsection{Characterizing equivaraint operators}
Generally, group\hyphen{}convolutional networks implement linear layers $L:L^1(X,\mu) \to L^1(Y,\nu)$ in the form of an \emph{integral operator}
$$(Lf)(y) = \int k(x,y) f(x) \mu(dx)$$
where $k:X\times Y\to \R$ is a $\mu$-a.e.\ $G$-invariant \emph{kernel function}:\footnote{The function $k$ is also known as the \emph{filter}. Conditions on $k$ may be needed to ensure $Lf$ is integrable.} $k$ is product-measurable, and for all $y \in Y$ and $g\in G$, we have $k(x,gy)=k(g\inv x, y)$ for $\mu$-a.e.\ $x\in X$. The equivariance of such integral operators $L$ is easily verified by first applying the invariance of the kernel function and then that of the integrating measure:
\begin{align*}
((Lf)\circ g)(y) &= \int k(x,gy) f(x) \mu(dx)= \int k(g\inv x,y) f(x) \mu(dx)\\ 
&= \int k(x,y) f(gx) \mu(dx)
= (L(f\circ g))(y).
\end{align*}
Equivariance of each entire layer --- and thus an entire network --- is guaranteed by letting the non-linearity $\sigma:L^1(Y,\nu)\to L^1(Y,\nu)$ in each layer be the ``pointwise'' application of a non-linearity $\sigma_Y: Y \to \R$
$$\sigma (Lf)(y) = \sigma_Y((Lf)(y)),$$
giving
$$(\sigma (Lf) \circ g)(y) = \sigma_Y( (Lf)(gy) ) = \sigma( (Lf) \circ g) (y).$$

A simplifying theoretical approach is to replace the combination of $\mu$ and $k$ above with a single (signed) \emph{transition kernel} $\kappa$. That is, for each $y \in Y$, we let $\kappa_y$ be a signed measure on $X$, such that the function $y\mapsto \kappa_y(A)$ is measurable for any measurable $A \subseteq X$. We may then define a \emph{kernel operator}\footnote{Again leaving aside conditions for integrability of the output.} $L:L^1(X,\mu) \to L^1(Y,\nu)$ by
$$(Lf)(y) = \int f d\kappa_y.$$
We say a kernel is \emph{$G$-invariant} if $\kappa_{gy} = \kappa_y \circ g\inv$ for all $g \in G$. It is easy to show that $L$ is $G$-equivariant if and only if $\kappa$ is $G$-invariant (see Lemma \ref{lem-inv-kernel} below).

Note that given a $G$-invariant kernel function $k$ and a $G$-invariant measure $\mu$, we can define an invariant transition kernel $\kappa_y(dx) = k(x,y)\mu(dx).$ The kernel operator and integral operator defined above then agree. On the other hand, if a $G$-invariant measure $\mu$ on $X$ exists and one restricts to kernels are that are ``nice,'' kernel operators correspond to integral operators: given an invariant kernel $\kappa$, if $\kappa_y$ is finite-valued and $\kappa_y\ll\mu$ for all $y\in Y$, one can define $x\mapsto k(x,y)$ as the Radon-Nikodym derivative $d\kappa_y/d\mu$.

The reason for considering transition kernels is that they provide an easy characterization of the equivariances of linear operators, provided they are in fact kernel operators (which, as mentioned above, is true for the linear maps usually used in group-equivariant convolutional networks). Indeed, for the result below, neither $X$ nor $Y$ need have a distinguished $G$-invariant measure $\mu$ or $\nu$.
\begin{lemma}
\label{lem-inv-kernel}
Consider a kernel operator $(Lf)(y) = \int f d\kappa_y$, and any two transformations $t_X \in \Aut(X)$ and $t_Y \in \Aut(Y)$. Then $(Lf) \circ t_Y = L(f \circ t_X)$ for all measurable $f$ (such that either side is defined) if and only if $\kappa_{t_Y(y)} = \kappa_y \circ t_X\inv$ for all $y \in Y$.
\end{lemma}
\begin{proof}
The reverse implication is trivial. For the forward implication it suffices to note that if two measures agree on which functions are integrable, and agree on the integrals of these functions, the measures are equal. (Otherwise, one could find a measurable set on which the measures disagree. The indicator function on such a set would have different integrals under the two measures, contradicting the hypothesis.)
\end{proof}
Fixing a measure, the above can be used to show that an equivariance of an integral operator which is an invariance of the measure is also an invariance of the kernel function.
\begin{corollary}
\label{cor-inv-kernel-function}
Consider an integral operator $(Lf)(y) = \int k(x,y) f(x) \mu(dx)$, and any two $t_X \in \Aut(X)$ and $t_Y \in \Aut(Y)$, such that $t_X$ preserves $\mu$. Then $(Lf) \circ t_Y = L(f \circ t_X)$ for all measurable $f$ (such that either side is defined) if and only if $k(g\inv x,y) = k(x,gy)$ for $\mu$-a.e.\ $x \in X$ for all $y \in Y$.
\end{corollary}
\begin{proof}
Again, the reverse implication is trivial. Assume then the equivariance of $L$. Defining $\kappa_y(dx) = k(x,y)\mu(dx),$ we see $L$ is the kernel operator $(Lf)(y) = \int f d\kappa_y$, so $\kappa$ is $(t_X, t_Y)$-invariant by the previous result. It follows $k$ is $\mu$-a.e.\ $(t_X, t_Y)$-invariant by the invariance of $\mu$ and the uniqueness statement of the Radon-Nikodym theorem. Explicitly, $k(x,t_Y y)=k(t_X\inv x,y)$ $\mu$-a.e.\ since, first by invariance of the kernel $\kappa$ and then that of $\mu$,
$$ \int k(x,t_Yy) f(x) \mu(dx) = \int k(x,y) f(t_Xx) \mu(dx) = \int k(t_X\inv x,y)f(x) \mu(dx). \qedhere$$
\end{proof}

One might ask which linear operators between function spaces are integral operators. If $X$ and $Y$ are finite the answer is of course \emph{all of them}, but the issue can become much more technical otherwise. \textcite{DunfordPettis1940,Bukhvalov1978, Schep1981} have proved relevant results, of which we refer to the following.
\begin{theorem}[{\cite[Corollary 2.3]{Schep1981}}]
If $L:L^1(X,\mu) \to L^1(Y,\nu)$ is weakly compact (takes bounded sets to relatively compact sets), then $L$ is an integral operator.
\end{theorem}
We obtain a corollary characterizing certain maps as having invariant kernel functions. Cohen has presented similar results for homogeneous spaces \parencite[Chapter 9]{Cohen2019,CohenPhd}. While in comparison we do not assume either $X$ or $Y$ is homogeneous, the aforementioned results are given in a more geometrically sophisticated context.
\begin{corollary}
\label{cor-inv-kernel}
Suppose $L:L^1(X,\mu) \to L^1(Y,\nu)$ is weakly compact. It is $G$-equivariant if and only if it is the integral operator given by a $\mu$-a.e.\ $G$-invariant kernel function $k$,
$$(Lf)(y) = \int k(x,y) f(x) \mu(dx).$$
\end{corollary}
\begin{proof}
Applying Corollary \ref{cor-inv-kernel-function}, the proof is immediate.
\end{proof}

It is often convenient to disintegrate (or \emph{factor} or \emph{decompose}) $\mu$ into an invariant measure on $G$ and a measure on the space $X/G$ of orbits. This is possible under rather general, though technical conditions (see \textcite{Eaton1989,Kallenberg2017}), for example when $G$ is compact second countable Hausdorff. In this case $G$ has a unique normalized measure which is left-invariant and right-invariant, the \emph{Haar measure} $\lambda$. Under the conditions of Corollary \ref{cor-inv-kernel} we then have
$$(Lf)(y) = \int k(x,y) f(x) \mu(dx) = \int\int k((g,o),y) f(g,o) \lambda(dg) \mu_{X/G}(do)$$
for a measure $\mu_{X/G}$ on the orbit space $X/G$. We may also write $y = (h,p)$ as an element of $G\times Y/G$. From the $G$-invariance of the kernel function (and of orbits) one then derives $k((h\inv g,o),(\id,p))= k((g,o),(h,p))$ $\mu$-a.e. Defining the function $\ell:G\times(X/G)^2\to\R$ with $\ell(g, o, p) = k((g,o),(\id,p))$, the operator $L$ then takes the form
$$(Lf)(h,p) = \int\int \ell(h\inv g, o, p) f(g,o) \lambda(dg) \mu_{X/G}(do).$$
The double integral above has appeared in \cite{Finzi2020} as the generalization of the group convolution of \cite{Cohen2016} to non-homogeneous spaces $X$: if $X=Y=G$ (or more generally, $G$ acts transitively on both $X$ and $Y$), then both $X$ and $Y$ contain only one orbit, so the operator above can be written as the classical group convolution
$$(Lf)(h) = \int \ell(h\inv g) f(g) \lambda(dg).$$

One may thus see Corollary \ref{cor-inv-kernel-function} and Corollary \ref{cor-inv-kernel} as versions of the theorem of Kondor and Trivedi \parencite[Theorem 1]{KondorTrivedi2018}, which says any linear equivariant map $L:L^1(X)\to L^1(Y)$ for spaces $X,Y$ acted on transitively by a compact group $G$ can be expressed in terms of a group convolution. Indeed, Lemma \ref{lem-inv-kernel} gives a version of the theorem for kernel operators, albeit not directly in terms an integral over $G$. It is notable that this fact has a fairly easy proof devoid of group-theoretic concerns --- unlike the representation-theoretic proof of \cite{KondorTrivedi2018}. Only in re-expressing a kernel operator as an integral over $G$ does some such care needs to be taken.

\subsubsection{Symmetry non-uniqueness}
Lemma \ref{lem-inv-kernel} characterizes the symmetry non-uniqueness of a kernel operator: the operator $(Lf)(y) = \int f(x)\kappa_y(dx)$ is equivariant exactly under the coupled actions of a group $T$ such that $\kappa_{ty} = \kappa_y\circ t\inv$ for all $y \in Y$ and $t \in T$. If $L$ is written as the integral operator
$(Lf)(y) = \int k(x,y) f(x) \mu(dx),$
this means $k(x,ty)\mu(dx) = k(t\inv x, y)(\mu\circ t\inv)(dx)$ for all $y \in Y$ and $t \in T$. Corollary \ref{cor-inv-kernel-function} shoes that if $\mu$ is $T$-invariant then $k(x,ty) = k(t\inv x, y)$ for $\mu$-a.e.\ $x$. That is, $L$ has exactly those symmetries encoded by the kernel function, at least where $\mu$ does not assign zero measure. Below, we will consider the case where $\mu$ assigns positive measure to ``non-trivial'' sets, namely, where $\mu$ is \emph{strictly positive} as a Borel measure: if $A \subseteq X$ is a non-empty open set, $\mu(A)>0$. Note that a product of two measures is strictly positive if and only if both measures are.

Given a distinguished measure $\mu$ on $X$, it is natural to restrict ourselves to those groups under which $\mu$ is invariant, when trying to learn a group. Following the reasoning above, one may hope the invariance of kernel functions restricts a learned group $T$ to essentially be $G$. We prove this is the case when $G$ is compact --- a convenient case, as it guarantees the disintegration of $\mu$, and that orbit spaces are Hausdorff (see \textcite[3.1 Theorem]{Bredon1972}). In other words, compact group convolution has good symmetry uniqueness properties.
\begin{remark}
For convenience in the proof, we consider operators with range $L^\infty(Y,\nu)$ rather than $L^1(Y,\nu)$. One may adapt the proof to a more natural setting, where the $L$ are operators between spaces of \emph{locally integrable} functions (those integrable on compact sets), provided $X$ and $Y$ are locally compact. In its current form, the result already applies to the work of \textcite{Zhou2021}, in which learnable finite-dimensional convolutions are implicitly restricted to be over permutation groups, which are invariances of the counting measure.
\end{remark}
\begin{theorem}
\label{thm-gcnn-uniqueness}
Let $X,Y$ be Hausdorff Borel spaces, and $G$ a compact Hausdorff group with Haar measure $\lambda$ and a continuous coupled action on $X,Y$ such that $X/G$ and $Y/G$ are regular. Let $\mu$ be a strictly positive $G$-invariant measure on $X$, and $\nu$ be a strictly positive measure on $Y$. Consider a group $T$ with continuous coupled action on $X,Y$, such that $\mu$ is $T$-invariant. The following are equivalent:
\begin{enumerate}[label=(\roman*)]
\item any $G$-equivariant integral operator $L:L^1(X,\mu) \to L^\infty(Y,\nu)$ is $T$-equivariant,
\item the coupled action of $T$ on $X,Y$ is that of a subgroup of $G$.
\end{enumerate}
\end{theorem}
\begin{proof}
If $T$ acts as subgroup of $G$ then the equivariance under $T$ of maps equivariant under $G$ is trivial. For the reverse implication it suffices to show that for any $t \in T$ there is $c_t \in G$ such that $t(x)=c_t x$ and $t(y)=c_t y$, and $c_s c_t = c_{st}$. Decomposing $X$ into the product of $G$ and its orbit space, we can with a slight abuse of notation write $t(x) = (t(g),t(o))$ for each $x=(g,o)$. We show that for any $x=(g,o)$ we have $t(x) = (t(g),t(o))= (c_t g, o)=c_tx$, where $c_t = t(\id)$. We show the same for $Y$. Note that $c_t = t(\id)$ implies $c_s c_t = c_{st}$.

The idea is that by Corollary \ref{cor-inv-kernel-function}, for any operator obtained as a convolution with $\ell(h\inv g, o, p)$ we have 
$$\ell(t(h)\inv g,o,t(p)) = \ell (h \inv t\inv (g), t\inv(o), p)$$
for $\mu$-a.e.\ $(g,o)$, and thus by varying $\ell$ one can show that $t(h)\inv g = h\inv t\inv(g)$, and $t(o)=o$ and $t(p) = p$. The work consists in showing this across all $(g,o)$, when it is only clear set-wise on sets of positive measure by letting $\ell$ be an indicator on such a set.\footnote{The resulting operator $L$ outputs a constant function, which is why we have $L^\infty(Y,\nu)$ as the range.} For $t(g) = c_t g$ it suffices to show that for any non-empty open $A\subseteq G$ we have $c_t \overline{A} = t(\overline{A})$. Analogously, for $t(o)=o$ it suffices to show that for each non-empty open $A\subseteq X/G$ we have $\overline{A} = t(\overline{A})$. (Of course we have the same for $Y/G$.) To see $c_t \overline{A} = t(\overline{A})$ suffices for the former, consider any $g \in G$, and recall $G$ is Hausdorff; for any closed neighborhood $V$ of $c_t g$, letting $U \subseteq V$ be an open set containing $c_t g$ we then have
$$t(g) \in t(c_t\inv U) \subseteq  \overline{t(c_t\inv U)} =  t(\overline{c_t\inv U}) = c_t \overline{c_t\inv U} = \overline{U} \subseteq V,$$
since homeomorphisms are both closed and open maps.\footnote{Letting $V$ ``go to zero'' around $c_t g$ proves $t(g) = c_tg$, where we use that regularity is equivalent to each point having a ``local basis'' of closed neighborhoods. We note that, being second countable, compact, and Hausdorff, $G$ is metrizable and thus regular.} Analogous reasoning holds for the orbit spaces. The sufficient conditions above are proved in Appendix \ref{appendix-conv-uniqueness}.
\end{proof}

\subsection{Semigroup convolution and homogeneous spaces}
\label{sec-semigroup-convolution}
We present two possible definitions of semigroup convolution: ``semigroup-like'' and ``group-like.'' For functions defined on the semigroup, only semigroup-like convolution seems to allow for a large class of symmetry non-uniqueness. If however one considers functions to be defined on a super-semigroup, both exhibit non-uniqueness. Finally, we argue no non-uniqueness arises in group-like convolution if functions are defined on a homogeneous space.

Recall the form of group convolution for a group $G$ with Haar measure $\lambda$,
 $$(Lf)(h) = \int \ell(h\inv g) f(g) \lambda(dg)$$
for any $f:G\to \R$ and $h \in G$, where $\ell:G\to \R$ is the kernel function. It is commonly noted that by the invariance of the measure $\lambda$, a ``change of variables'' gives
\begin{align}
\label{eq-semigroup-same}
(Lf)(h) = \int \ell(g) f(hg) \lambda(dg).
\end{align}
Since the expression above contains no inverses, the map $L$ can be defined for semigroups, rather than just groups. \textcite{Worrall2019} define semigroup convolution differently, with the order of the product in $f$ reversed:
\begin{align}
\label{eq-semigroup-different}
(Lf)(h) = \int \ell(g) f(gh) \lambda(dg).
\end{align}
To distinguish between the integrals, we call (\ref{eq-semigroup-same}) \emph{group-like convolution} and (\ref{eq-semigroup-different}) \emph{semigroup-like convolution}. We consider their symmetry non-uniquenesses, where the action on functions $f$ is induced by a right action on the underlying set of the semigroup. The standard action of a group $G$ on functions $f:G\to \R$, given by $f\mapsto f\circ g\inv$ for each $g \in G$, provides an example of a right action inducing a left action on functions. Namely the action $\alpha(g)(h) = g\inv h$ is a right action of $G$ on $\set(G)$. Similarly, we can induce a left action on functions $f:S\to \R$ given any right action $\alpha:T^\text{opp}\to\End(\set(S))$ of a semigroup $T$ on $S$, by $t\cdot f = f\circ\alpha(t)$. We can write $t\cdot f = f \circ t$ if we remember that $t$ is acting on the right, so $((t_1t_2)\cdot f)(s) = f (t_2 (t_1( s))).$
\begin{proposition}
Let $S$ be a semigroup, with a measure $\lambda$ (not necessarily invariant), and let $T$ be a semigroup with a right action on $\set(S)$. A group-like convolution
\begin{align*}
(Lf)(s_1) = \int \ell(s_2) f(s_1s_2) \lambda(ds_2) \tag{\ref{eq-semigroup-same} revisited}
\end{align*}
is equivariant under the induced left action of $T$ on functions, if $t(s_1s_2) = t(s_1)s_2$ for all $t \in T$ and $s_1, s_2 \in S$.
A semigroup-like convolution 
\begin{align*}
(Lf)(s_1) = \int \ell(s_2) f(s_2s_1) \lambda(ds_2) \tag{\ref{eq-semigroup-different} revisited}
\end{align*}is equivariant under the induced left action of $T$ on functions, if $t(s_2s_1) = s_2t(s_1)$ for all $t \in T$ and $s_1, s_2 \in S$.
\end{proposition}
Note that under the right conditions on $S$ and $\lambda$, the proof below shows the conditions for equivariance of (\ref{eq-semigroup-same}) and (\ref{eq-semigroup-different}) are both sufficient and necessary (up to null sets). One may wish to compare this result with Theorem \ref{thm-gcnn-uniqueness}.
\begin{proof}
We write down the proof for group-like convolution, the proof for semigroup-like convolution being entirely similar:
\begin{align*}
(L(t\cdot f))(s_1) 	&= \int \ell(s_2) (t\cdot f)(s_1s_2) \lambda(ds_2) = \int \ell(s_2) f(t(s_1s_2)) \lambda(ds_2) \\
				&=  \int \ell(s_2) f(t(s_1)s_2) \lambda(ds_2) = (Lf)(t(s_1)) = (t \cdot Lf)(s_1).\qedhere
\end{align*}
\end{proof}
One reading of the result above is that (\ref{eq-semigroup-different}) is a more natural definition for semigroup convolution (justifying calling it semigroup-like). Letting $T=S$ be the semigroup acting on its own underlying set by multiplication on the right, for any $t \in S$ we have $t(s_2s_2) = s_2s_1 t = s_2 t(s_1)$, and thus (\ref{eq-semigroup-different}) is equivariant (as noted in \cite{Worrall2019}). Right multiplication in $S$ does not satisfy the equivariance condition for (\ref{eq-semigroup-same}): $t(s_1s_2) =s_1 s_2t \ne s_1ts_2= t(s_1)s_2$.

On the other hand, the equivariance condition for (\ref{eq-semigroup-different}) does not generally imply that $T=S$. It in fact suffices that $T$ be a super-semigroup of $S$, such that $S$ is a \emph{right ideal} of $T$ ($st \in S$ for all $s \in S$ and $t \in T$). That $S$ is a right ideal is needed to ensure that $T$ acts on $S$, so the function $f:S\to\R$ remains defined when acted upon by $t \in T$. If signals are instead $f:T\to\R$, then no such condition is required. Note that if $S$ acts transitively on itself on the right, then Proposition \ref{prop-transitive-group} shows $T=S$, so no non-uniqueness arises.

Whether the above result gives concrete non-uniquenesses for group-like convolution is delicate. The one example we have on hand satisfying the equivariance condition for (\ref{eq-semigroup-same}) is that of groups $S=G\leq T$ with the right action $t: g\mapsto t\inv g$. The corresponding ``ideal'' condition, that $t(g)=t\inv g \in G$ for all $t \in T$ and $g \in G$, implies immediately that $T \subseteq G$ and thus $G=T$. Thus if working directly over a group $G$, no non-uniqueness apparently arises. As above, however, if signals are intrinsically defined on $T$ then a group-like convolution over $G$ will always be $T$-equivariant:
\begin{align*}
(L(t\cdot f))(g_1) = \int \ell(g_2) f(t\inv g_1g_2) \lambda(dg_2)  = (Lf)(t\inv g_1) = (t\cdot (Lf))(g_1).
\end{align*}

Care is further required in studying group-like convolution over a homogeneous space $X$, as used for example by \textcite{Dehmamy2021}. Fixing a ``reference point'' $x_0 \in X$, we can write $x = g_x x_0$ with some $g_x \in G$ for any $x \in X$. For any scalar function $f$ defined on $X$, we then define the group-like convolution (assuming well-definedness)
$$(Lf)(x) = \int \ell(g) f(g_xgx_0) \lambda(dg).$$
The standard idea that $X$ can be identified with a coset space (see for example \parencite{KondorTrivedi2018}), together with our above analysis for group-like convolutions of signals $f:G\to \R$, suggests no non-uniqueness arises. We do not prove this fact, but exhibit an alternative line of reasoning which may be turned into a proof in certain cases.  A group $T$ with right action $x\mapsto t\inv x$ gives rise to the symmetry non-uniqueness
\begin{align*}
(L(t\cdot f))(x) 	&= \int \ell(g) (t\cdot f)(g_xgx_0) \lambda(dg) = \int \ell(g) f(t\inv g_x g x_0) \lambda(dg) \\
			&=  \int \ell(g) f(g_{t\inv x} g x_0) \lambda(dg) = (Lf)(t\inv x) = (h\cdot (Lf))(x)
\end{align*}
if (and possibly only if) $t\inv g_x g x_0 = g_{t\inv x} g x_0$ for all $x\in X, g \in G$ and $t \in T$. For abelian $G$ (like those learned in \parencite{Dehmamy2021}), it is easy to show this holds if and only if $G\leq Z(T)$. But Proposition \ref{prop-transitive-group} then implies that $G$ and $T$ act as the same group on $X$.

\section{Non-uniqueness maximality}
\label{sec-max}
We explore a notion of maximality for symmetry non-uniqueness. We begin in Section \ref{sec-max-motivation} by noting an example thereof encountered above, then providing a general definition and some motivation. We show that in the setting of linearly equivariant networks (Section \ref{sec-lin-eq}), if the symmetries $\Gamma$ (and corresponding actions on network layers) consist of semisimple representations of groups whose irreducible subrepresentations are finite-dimensional, then up to some degeneracy the maximal non-uniqueness is given by taking the group algebra. We demonstrate this for linear maps in Section \ref{sec-maximality-linear}. Extending the result to linearly equivariant networks which follow the design pattern of ``universally equivariant non-linearities'' is trivial, since as mentioned in the introduction to Section \ref{sec-ansatzes}, these ``non-linearities'' are trivial maps. Other cases may require implementation-specific details which we do not treat here. In Section \ref{sec-maximality-type-I} we show (assuming the axiom of choice) the result can be extended to all unitarizable representations of so-called groups of type I, which includes all second countable Hausdorff groups that are compact, or abelian and locally compact. 

\subsection{Motivation}
\label{sec-max-motivation}
In the previous section, equivariant neural networks are studied with an eye towards identifying when they are equivariant under more symmetries than intended. In the case of neural networks with linear layers given by kernel operators, Lemma \ref{lem-inv-kernel} allows us to identify equivariances \emph{exactly}: they are simply the invariances of the kernels. A neural network with linear layers of this form (and pointwise non-linearities) will thus exhibit layer-wise equivariance precisely under those transformations which leave the kernels invariant. Suppose an equivariant ansatz consists of pairs $(f,\gamma)$ where $\gamma$ is a coupled group action and $f$ is --- borrowing our notation from linearly equivariant networks --- a $(\gamma)$-equivariant group\hyphen{}convolutional neural network, meaning the kernels in the linear layers are somehow restricted to be invariant under corresponding coupled actions. One possible symmetry non-uniqueness in this case is to map $\gamma$ to the coupled action of the largest group under which the kernels remain invariant. In fact, this is in a sense the most general non-uniqueness for such networks, and gives us a first example of maximality of a symmetry non-uniqueness.

Recall the setting of Theorem \ref{thm-main-thm}: we have a set of objects $F$, symmetries $\Gamma$, a relation $R\subseteq F\times \Gamma$ and ansatz $A\subseteq R$. Finally, a function $H:\Gamma\to\Gamma$ is required to satisfy certain conditions (either the conditions for (\ref{eq-subset}) or those for (\ref{eq-equal})). Let us introduce a  preorder $\preccurlyeq$ on $\Gamma$ (a reflexive and transitive relation), with $\gamma_1 \preccurlyeq \gamma_2$ when $(f,\gamma_2) \in R$ implies $(f,\gamma_1) \in R$.\footnote{In the case of semigroups, containment as sub-semigroups is a preorder satisfying this condition. In general, however, the condition itself defines a preorder weaker than containment.} Fixing a particular set of conditions on functions $H$, we say a particular function $H_*$ is \emph{maximal} if for any $H$ satisfying the said conditions, $H(\gamma) \preccurlyeq H_*(\gamma)$ for all $\gamma \in \Gamma_A$.

A reason to be interested in a maximal non-uniqueness is that if it exists, it characterizes the failures of the ansatz in question. For example if $H_*$ is maximal for the conditions of (\ref{eq-subset}), and $(f,H_*(\gamma)) \in R$ for all $(f,\gamma)\in A$, then the conditions of Corollary \ref{cor-unlearnable} --- which implies the existence of a non-learnable pair $(f,\gamma)\in R$ --- are never satisfied. Perhaps more usefully, the contrapositive of Corollary \ref{cor-unlearnable} implies that for any learnable symmetric pair $(f,\gamma)$ one can expect that using the ansatz to try to learn $f$ without knowing $\gamma$, one will learn a pair $(f,\gamma')$ where $\gamma' \preccurlyeq H_*(\gamma)$. In terms of the learnability tradeoff (Corollary \ref{cor-learnability-tradeoff}), the ``worst'' symmetry not distinguishable from $\gamma$ is $H_*(\gamma)$ where $H_*$ is maximal for the conditions of (\ref{eq-equal}).

In Section \ref{sec-lin-eq} we show the semigroup algebra  is \emph{one possible} symmetry non-uniqueness of linearly equivariant networks. In unpublished experiments, however, linearly equivariant networks learn elements of the semigroup algebra rather than the semigroup. A natural question follows from this observation: is the semigroup algebra the maximal non-uniqueness? As we show in the next couple of sections, the answer is ``yes, up to some degeneracy.''

\subsection{Maximality for linear maps}
\label{sec-maximality-linear}
Recall that Proposition \ref{prop-schur} states that for any semisimple representations $\rho_U$ and $\rho_V$ of a group $G$, whose irreducible components satisfy the ``strong version'' of Schur's lemma, the $(\rho_U,\rho_V)$-equivariant linear maps $L:U\to V$ are exactly those of the form $\bigoplus_{k} L_k$ where each $L_k : W_k^{\oplus n_k} \to W_k^{\oplus m_k}$ acts as a $m_k \times n_k$ matrix of scalars; here, $U = \bigoplus_{k} W_k^{\oplus n_k}$ and $V = \bigoplus_{k} W_k^{\oplus m_k}$ where each $W_k$ has irreducible representation $\rho_k$. This characterization allows us to describe exactly which pairs of transformations $t_U \in \End(U)$ and $t_V\in\End(V)$ are such that $t_V L = L t_U$ for all $(\rho_U,\rho_V)$-equivariant $L$.
\begin{proposition}
\label{prop-maximality-linear}
Let $U = \bigoplus_{k} W_k^{\oplus n_k}$ and $V = \bigoplus_{k} W_k^{\oplus m_k}$ with representations $\rho_U,\rho_V$ given naturally from irreducible representations $\rho_k:G\to\GL(W_k)$ satisfying the strong version of Schur's lemma. Consider $t_U \in \End(U)$ and $t_V\in\End(V)$, decomposing the transformations into ``matrix'' components $(t_U)_{k\ell} : W_k^{\oplus n_k} \to W_\ell^{\oplus n_\ell}$ and $(t_V)_{k\ell} : W_k^{\oplus m_k} \to W_\ell^{\oplus m_\ell}$. Then $t_V L = L t_U$ for any $(\rho_U,\rho_V)$-equivariant $L:U\to V$ if and only if each of the following holds:
\begin{enumerate}[label=(\roman*)]
\item if $m_k\ne0$ and $n_k\ne0$, then as $m_k\times m_k$ and $n_k \times n_k$ ``matrices'' of entries $W_k\to W_k$, respectively, $(t_V)_{kk}$ and $(t_U)_{kk}$ are diagonal, with $(t_V)_{kk,ii}=(t_U)_{kk,jj}$ for any $i \in [m_k]$ and $j\in[n_k]$,
\item for all $k \ne \ell$ such that $m_k, n_k, n_\ell$ are all nonzero, we have $(t_U)_{k\ell} = 0$,
\item for all $k \ne \ell$ such that $m_k,m_\ell, n_\ell$ are all nonzero, we have $(t_V)_{k\ell} = 0$.
\end{enumerate}
\end{proposition}
\begin{proof}
That pairs $t_U,t_V$ of the form above satisfy $t_V L = L t_U$ for all $(\rho_U,\rho_V)$-equivariant $L$ follows directly from the form of $L$ given by Proposition \ref{prop-schur}. Namely, it suffices to see $(t_V)_{kk} L_k = L_k(t_U)_{kk}$, by the same sort of algebra as in the proof of the aforementioned result. (Note the accompanying discussion of the cases where either $m_k$ or $n_k$ is zero.)

Suppose then that $t_U$ and $t_V$ are such that  $t_V L = L t_U$ for all $(\rho_U,\rho_V)$-equivariant $L$. We prove each of the points above holds in turn. 

Suppose $m_k\ne0$ and $n_k\ne0$ and consider $(t_U)_{kk}$ and $(t_V)_{kk}$ respectively as $n_k\times n_k$ and $m_k \times m_k$ matrices of entries $W_k \to W_k$. By Proposition \ref{prop-schur}, $L_k$ can be taken as any $m_k \times n_k$ matrix of scalar multiples of the identity on $W_k$.  We use this fact to first show $(t_V)_{kk,ii}=(t_U)_{kk,jj}$ for all $i \in [m_k]$ and $j\in[n_k]$. Let $L_k$ be the identity in the $(i,j)$-th entry and zero elsewhere. By the assumption that $t_V L = L t_U$ we have $((t_V)_{kk} L_k)_{ij} = (L_k(t_U)_{kk})_{ij}$. Expanding both sides we observe
\begin{align*}
((t_V)_{kk} L_k)_{ij} &= \sum_{p =1}^{ m_k} (t_V)_{kk,ip}L_{k,pj} = (t_V)_{kk,ii}\\
(L_k (t_U)_{kk})_{ij} &= \sum_{q=1}^{ n_k} L_{k,iq} (t_U)_{kk,qj} = (t_U)_{kk,jj}.
\end{align*}
This same choice of $L_k$ shows that for any $i' \ne i$ and $j' \ne j$, we have $(t_V)_{kk,i'i} = 0$ and $(t_U)_{kk,jj'}=0$. In particular,
\begin{align*}
((t_V)_{kk} L_k)_{i'j} &= \sum_{p =1}^{ m_k} (t_V)_{kk,i'p}L_{k,pj} = (t_V)_{kk,i'i}\\
(L_k (t_U)_{kk})_{i'j} &= \sum_{q=1}^{ n_k} L_{k,i'q} (t_U)_{kk,qj} = 0,
\end{align*}
and similarly
\begin{align*}
((t_V)_{kk} L_k)_{ij'} &= \sum_{p =1}^{ m_k} (t_V)_{kk,ip}L_{k,pj'} = 0\\
(L_k (t_U)_{kk})_{ij'} &= \sum_{q=1}^{ n_k} L_{k,iq} (t_U)_{kk,qj'} = (t_U)_{kk,jj'}.
\end{align*}
Thus the first point is proved.

Next consider any $k \ne \ell$, such that $m_k, n_k, n_\ell$ are all non-zero. By assumption we have $(t_V L)_{k \ell}=(L t_U)_{k\ell} $, where the right side can be rewritten as
$$(L t_U)_{k\ell} = \sum_{k'} L_{kk'}(t_U)_{k'\ell} = L_k (t_U)_{k\ell}.$$
If $m_\ell =0$ then $(t_V L)_{k \ell}=0$, since $W_\ell^{n_\ell}$ must be in the kernel of $L$. If on the other hand $m_\ell \ne 0$ we may write
$$(t_V L)_{k \ell}=\sum_{k'} (t_V)_{kk'} L_{k' \ell}=(t_V)_{k\ell} L_{\ell}.$$
Let $L$ be such that $L_\ell = 0$ and $L_k$ is all zero except an identity at entry $(i,j)$. Then on one hand $L_k (t_U)_{k\ell} = (t_V)_{k\ell} L_{\ell} = 0$, and on the other
$$(L_k (t_U)_{k\ell})_{ij'} = \sum_{q=1}^{ n_k} L_{k,iq} (t_U)_{k\ell,qj'} = (t_U)_{k\ell,jj'}$$
for each $j' \in [n_\ell]$. Since this holds for any $j \in [n_k]$ we have that $(t_U)_{k\ell}=0$. The second point is proved.

The proof of the third point should by now be clear. Consider $k \ne \ell$ such that $m_k, m_\ell, n_\ell$ are all non-zero. If $n_k = 0$ then $(L t_U)_{k\ell} = 0$ since the image of $L$ in the $W_k^{\oplus m_k}$ component of $V$ must be zero. Supposing otherwise, we can essentially repeat the argument above.
\end{proof}
There remain a couple points to clarify. First is how the result, which is about symmetry non-uniqueness generally, relates to the group algebra. Second, we ought to make clear a certain degeneracy which implies this correspondence is not exact. To conclude this section, both these points are illustrated with an example.

The correspondence with the group algebra can be seen as follows. Note that if $m_k$ and $n_k$ are all non-zero, the result above says that the pairs $t_U, t_V$ are exactly those of the form $t_U = \bigoplus_k t_k^{\oplus n_k}$ and $t_V = \bigoplus_k t_k^{\oplus m_k}$ for transformations $t_k \in \End(W_k)$. Thus the set of such pairs can be identified with the single transformations $t = \bigoplus_k t_k \in \bigoplus_k \End(W_k)$. If the irreducible representations are all finite-dimensional and on complex vector spaces, then it turns out that $\bigoplus_k \End(W_k)$ is the group algebra of the image of the representation $\bigoplus_k \rho_k$. This follows from the fact that $\End(W_k)$ is the group algebra of the image of $\rho_k$, by Proposition \ref{prop-density}. Therefore any pair $t_U, t_V$ can actually be obtained as an element of $\C[\rho(G)]$ where $\rho=(\rho_U,\rho_V)$ is the appropriate coupled representation.

The degeneracy in the result above comes from the fact that $(t_U)_{k\ell}$ is unconstrained when $m_k = 0$, and similarly $(t_V)_{k\ell}$ can be anything when $n_\ell = 0$. Thus the space of pairs $t_U, t_V$ which serve as symmetries for $(\rho_U, \rho_V)$-equivariant maps $L$ cannot exactly be identified with the group algebra, unless every irreducible representation occuring in $U$ also occurs in $V$, and vice versa.

\begin{example}
\label{ex-perm}
Consider the permutation group $S_3$. Its irreducible representations are:
\begin{enumerate}[label=(\roman*)]
\item $W_0 = \C$ with the trivial representation $\rho_0(\sigma)=\id$, 
\item $W_1 = \C$ with $\rho_1(\sigma) = \sign(\sigma)$,
\item $W_2 = \{v \in \C^3 : \sum_i v_i = 0\}$ with $(\rho_2(\sigma)v)_i = v_{\sigma(i)}$.
\end{enumerate}
(Note that $W_2$ is two dimensional, and the representation is usually written in terms of two basis vectors of $W_2$.) Let $U = W_0 \oplus W_2$ and $V = W_1 \oplus W_2$, with $\rho_U,\rho_V$ the corresponding representations. We find the general form of an equivariant map $L:U\to V$, and from this derive which transformations $t_U,t_V$ satisfy $t_VL=Lt_U$ for all such $L$.

To express an equivariant linear map $L:U\to V$ as a matrix we need to choose bases for $U$ and $V$. It is not hard to see that $U$ is in fact $\C^3$ with the usual representation $(\rho_U(\sigma)v)_i = v_{\sigma(i)}$ of $S_3$ in the standard basis, by noting that the span of the sum of the standard basis vectors is $W_0$ and its orthogonal complement is $W_2$. That is, letting $u_1 = (1,1,1)$ in the standard basis, and similarly $u_2 = (1,-1,0)$ and $u_3=(0,1,-1)$, we have that $u_1$ is a basis for $W_0$ in $U$, and $\{u_2, u_3\}$ a basis for $W_2$. We can of course move between the standard basis and $(u_1,u_2,u_3)$ with the change of basis matrix
$$P=\begin{pmatrix}
1 & 1 & 0\\
1 & -1 & 1 \\
1 & 0 & -1
\end{pmatrix},$$
with inverse
$$P\inv=\frac{1}{3}\begin{pmatrix}
1 & 1 &1 \\
2 &-1 & -1 \\
1 & 1 & -2
\end{pmatrix}.$$
Using this, one can represent each of the six elements of $S_3$ in the $(u_1,u_2,u_3)$ basis using the map $\sigma\mapsto P\inv \rho_U(\sigma)P$, giving
\begin{align*}
&&\id \mapsto
\begin{pmatrix}
1 & 0 & 0\\
0 & 1 & 0 \\
0 & 0 & 1 \\
\end{pmatrix}
&&(1~2~3) \mapsto
\begin{pmatrix}
1 & 0 & 0\\
0 & 0 & -1 \\
0 & 1 & -1 \\
\end{pmatrix}
&&(1~3~2) \mapsto
\begin{pmatrix}
1 & 0 & 0\\
0 & -1 & 1 \\
0 & -1 & 0 \\
\end{pmatrix}\\
&&(1~2) \mapsto
\begin{pmatrix}
1 & 0 & 0\\
0 & -1 & 1 \\
0 & 0 & 1 \\
\end{pmatrix}
&&(2~3) \mapsto
\begin{pmatrix}
1 & 0 & 0\\
0 & 1 & 0 \\
0 & 1 & -1 \\
\end{pmatrix}
&&(1~3) \mapsto
\begin{pmatrix}
1 & 0 & 0\\
0 & 0 & -1 \\
0 & -1 & 0 \\
\end{pmatrix}.
\end{align*}
The block structure of the matrices makes the decomposition of $U$ as $W_0\oplus W_2$ clear. Note that the bottom right $2\times 2$ submatrices span the set of all $2 \times 2$ matrices. This corresponds to the fact that $W_2$ is irreducible and finite-dimensional, so the corresponding group algebra $\C[\rho_2(S_3)]$ is all of $\End(W_2)$.

To make things interesting, we let $\rho_V$ act on the standard basis for $V=\C^3$ by letting the same matrix $P$ move from the standard basis on $V$ to a basis $(v_1,v_2,v_3)$ on which the permutations act as follows:
\begin{align*}
&&\id \mapsto
\begin{pmatrix}
1 & 0 & 0\\
0 & 1 & 0 \\
0 & 0 & 1 \\
\end{pmatrix}
&&(1~2~3) \mapsto
\begin{pmatrix}
1 & 0 & 0\\
0 & 0 & -1 \\
0 & 1 & -1 \\
\end{pmatrix}
&&(1~3~2) \mapsto
\begin{pmatrix}
1 & 0 & 0\\
0 & -1 & 1 \\
0 & -1 & 0 \\
\end{pmatrix}
\end{align*}
\begin{align*}
&&(1~2) \mapsto
\begin{pmatrix}
-1 & 0 & 0\\
0 & -1 & 1 \\
0 & 0 & 1 \\
\end{pmatrix}
&&(2~3) \mapsto
\begin{pmatrix}
-1 & 0 & 0\\
0 & 1 & 0 \\
0 & 1 & -1 \\
\end{pmatrix}
&&(1~3) \mapsto
\begin{pmatrix}
-1 & 0 & 0\\
0 & 0 & -1 \\
0 & -1 & 0 \\
\end{pmatrix}.
\end{align*}
Note this is the same table as above, but where for the odd permutations (those in the second row), the top left entry is $-1$. This corresponds to the fact that $v_1$ is the basis vector for $W_1$ in $V$, and $\rho_1$ is the sign representation. Of course, $\{v_2,v_3\}$ forms the basis of the copy of $W_2$ in $V$. Interestingly, denoting by $(e_1,e_2,e_3)$ the standard basis on $U=V=\C^3$, one can show that $$\inner{e_i}{\rho_V(\sigma)e_j} = \inner{e_i}{\rho_U(\sigma)e_j} - \frac{2}{3}$$ for all $\sigma \in S_3$ and $i,j\in[3]$. That is, in the standard basis, $\rho_U$ and $\rho_V$ act as matrices that are the same except for an entry-wise subtraction.

We can now ask what are the linear maps $L:U\to V$ which are $(\rho_U, \rho_V)$-equivariant. By Proposition \ref{prop-schur}, we know that such maps $L$, as a matrices in the $(u_1,u_2,u_3)$ and $(v_1,v_2,v_3)$ bases, are exactly those of the form
$$\begin{pmatrix}
0 & 0 & 0 \\
0 & c & 0 \\
0 & 0 & c
\end{pmatrix}$$
for some $c \in \C$. Passing to the standard basis, we obtain
$$ L = \frac{c}{3}\begin{pmatrix}
2 & -1 & -1 \\
-1 & 2 & -1 \\
-1 & -1 & 2
\end{pmatrix}.$$
To understand which maps $t_U \in \End(U)$ and $t_V \in \End(V)$ satisfy $t_V L = L t_U$ for all such $L$, however, it is convenient to stay in the $(u_1,u_2,u_3)$ and $(v_1,v_2,v_3)$ bases, as in the proof of Proposition \ref{prop-maximality-linear}. In particular, it is clear that in the $(u_1,u_2,u_3)$ and $(v_1,v_2,v_3)$ bases, we have that $t_U$ is of the form
$$\begin{pmatrix}
x_1 & x_2 & x_3\\
0 & a & b\\
0 & c & d
\end{pmatrix}$$
and $t_V$ of the form
$$\begin{pmatrix}
y_1 & 0 & 0\\
y_2 & a & b\\
y_3 & c & d
\end{pmatrix},$$
for some scalars $x_i,y_i,a,b,c,d\in\C$. Note that the bottom right $2\times 2$ submatrices are identical, and are of course in $\C[\rho_2(S_3)]$. The freedom of the first row in the matrix for $t_U$ comes from the fact that $W_0$ does not occur in $V$, and the freedom of the first column in $t_V$ from the fact that $W_1$ does not occur in $U$. For this reason we say that $t_U$ and $t_V$ are elements of the group algebra up to a degeneracy, arising from irreducible representations occurring in only one of $U$ or $V$.  If we had replaced $W_1$ in $V$ with $W_0$, so $U=V$, the form of $L$ in the $(u_1,u_2,u_3)$ and $(v_1,v_2,v_3)$ bases would have a free scalar in the top left entry; we would have then found that in the standard basis $L$ can be any matrix of the form 
$$ L = \begin{pmatrix}
a & b & b \\
b & a & b \\
b & b & a
\end{pmatrix}.$$
More importantly $t_U$ and $t_V$ would be equal, of the form above but with $x_1=y_1$ as any scalar and $x_2,x_3,y_2,y_3$ all zero; that is, they would be elements of $\C[(\rho_0\oplus\rho_2)(G)]$.
\end{example}

\subsection{Unitary representations of groups of type I}
\label{sec-maximality-type-I}
We now turn to the case of groups of type I. The result is similar to the semisimple case: up to the same type of degeneracy, the maximal symmetry non-uniqueness is given by taking the group von Neumann algebra (defined below). In comparison to our previous discussions, this is a rather technical notion, due to the fact that we replace the direct sum decomposition of semisimple representations with a direct integral. We break this section up into four pieces, all but the last dedicated to briefly defining relevant notions, with the last finally presenting a proof of maximality of a symmetry non-uniqueness. We rely heavily on the book of \textcite{Bekka2019} on the subject of unitary representations.

For the remainder of the section, whenever we refer to a representation we mean a unitary representation. Similarly, equivalence of representations refers to unitary equivalence.

\subsubsection{Von Neumann algebras}
The correct generalization of the group algebra we encounter below is a \emph{von Neumann algebra}: a subalgebra $\mathcal{M}$ of the algebra $\mathcal{L}(V)$ of bounded linear operators on a Hilbert space $V$, which is closed under taking adjoints and such that  it is equal to its \emph{bicommutant}: $\mathcal{M} = \mathcal{M}''$. The \emph{commutant} of a subset $S\subseteq \mathcal{L}(V)$ is the set $S' = \{T \in \mathcal{L}(V): TL = LT ~\forall L \in S\}$. Note that $S_1\subseteq S_2$ implies $S_2'\subseteq S_1'$, so $S'=S'''$ is a von Neumann algebra for any $S$ closed under taking adjoints. In particular, for a unitary representation $\pi$ of a group $G$, we have that $\pi(G)'$ is a von Neumann algebra. Similarly, so is $\pi(G)''$, which is the \emph{group von Neumann algebra} generated by the group $\pi(G)$.\footnote{The group von Neumann algebra is related to the group algebra $K[G]$, in the sense that it arises as a completion of a convolution algebra of functions. In particular, one obtains $\pi(G)''$ from a sequence of completions of the functions $L^1(G,\lambda)$ (where $\lambda$ is a Haar measure), rather than functions with finite support as in the case of $K[G]$. Of course, if $G$ is finite these coincide, as do $\pi(G)''$ and $K[G]$.}

The celebrated von Neumann bicommutant theorem states that a subalgebra $\mathcal{M}$ of $\mathcal{L}(V)$ which is closed under taking adjoints and contains the identity is a von Neumann algebra if and only if it is closed in the strong (or equivalently, weak) operator topology \parencite[Appendix K]{Bekka2019}. As an example application, this result provides an alternative proof of Proposition \ref{prop-maximality-linear} for finite-dimensional unitary representations.

\subsubsection{Direct integrals}
We begin by defining the direct integral of Hilbert spaces, and operators thereon. While this material is standard, our presentation specifically follows that of \cite[Section 1.G]{Bekka2019}, to which we recommend the reader refer if possible.

Let $(X,\Sigma,\mu)$ be a $\sigma$-finite measure space. Consider a collection of Hilbert spaces $(V_x)_{x \in X}$ together with collections of vectors $(e_{x,n})_{n\in\N} \in V_x$ with span dense in $V_x$ for every $x \in X$, such that $x \mapsto \inner{e_{x,n}}{e_{x,m}}$ is measurable for all $n,m\in\N$. We call $v \in \prod_{x \in X} V_x$ a \emph{measurable vector field} if $x \mapsto \inner{v_x}{e_{x,n}}$ is measurable for every $n \in \N$. We consider these modulo equality on sets of positive measure. Defining the inner product $\inner{v}{w} = \int_X \inner{v_x}{w_x} \mu(dx)$, the set of square-integrable measurable vector fields is the \emph{direct integral}
$$V = \int_X^\oplus V_x \mu(dx).$$
Note that a direct sum is a special case, given by taking $\mu$ to be the counting measure.

Given two direct integrals of Hilbert spaces, $U$ and $V$, and operators $T_x : U_x \to V_x$ such that $x\mapsto \inner{T_xu_x}{v_x}$ is measurable and $x\mapsto ||T_x||$ is $\mu$-essentially bounded ($\mu$-almost equal to a bounded function), we define the operator
$$T = \int_X^\oplus T_x \mu(dx)$$
by $(Tu)_x = T_x u_x$. Operators of this form are called \emph{decomposable}. As a special case, we have for $\mu$-essentially bounded functions $\varphi:X\to\C$ the \emph{diagonalizable operators}
$$m(\varphi) = \int_X^\oplus \varphi(x)\id_{V_x}\mu(dx).$$
For example, if the direct integral is a direct sum of finite-dimensional vector spaces, diagonalizable operators are exactly the direct sums of scalar multiples of identity matrices.
Another case of particular interest is when the $T_x$ above are of the form $\pi_x(g)$ where $\pi_x:G\to \mathcal{U}(V_x)$ are representations, in which case one defines in the obvious way the unitary representation
$$\pi = \int_X^\oplus \pi_x \mu(dx).$$

Under certain measurability conditions \parencite[Section 1.I]{Bekka2019}, given a collection of von Neumann algebras $(\mathcal{M}_x)_{x\in X}$ one defines $\int_X^\oplus \mathcal{M}_x\mu(dx)$ to be the von Neumann algebra of all decomposable operators $T=\int_X^\oplus T_x \mu(dx)$ such that $T_x \in \mathcal{M}_x$ for $\mu$-almost every $x\in X$. Such a von Neumann algebra is by analogy called \emph{decomposable}. We note the measurability conditions are satisfied by a collection of group von Neumann algebras $(\pi_x(G)'')_{x\in X}$, as well as any collection $(\mathcal{M}_x')_{x\in X}$ for von Neumann algebras $(\mathcal{M}_x)_{x\in X}$.

\subsubsection{Type I representations}
We review the formal definition of certain relations between representations, which are used to define type I representations. We follow \cite[Section 6.A]{Bekka2019}.

We call representations $\rho$ and $\pi$ \emph{disjoint} if they have no non-trivial equivalent subrepresentations, writing $\rho\perp \pi$. We say $\rho$ is \emph{subordinate} to $\pi$, writing $\rho\lesssim\pi$, if for any non-trivial subrepresentation $\rho' \le \rho$, it is not the case that $\rho' \perp \pi$. If $\rho$ and $\pi$ are subordinate to each other, we call them \emph{quasi-equivalent}, writing $\rho \approx \pi$. We call a representation $\rho$ \emph{multiplicity-free} if $\rho = \rho_1\oplus\rho_2$ implies $\rho_1\perp\rho_2$. A representation which is quasi-equivalent to a multiplicity-free representation is called \emph{type I}. A \emph{group of type I} is a group whose (unitary) representations are all of type I.
\begin{remark}
As made clear below, quasi-equivalence of representations can be seen as a condition avoiding the degeneracies discussed in the case of semisimple representations: essentially, any irreducible representation occurring in one occurs in the other. We encounter this case above in our discussion following Proposition \ref{prop-maximality-linear}, where a multiplicity-free representation arises, in particular the direct sum of the irreducible representations $\rho_k$.
\end{remark}

\subsubsection{Symmetry non-uniqueness}
Suppose now that $\rho$ and $\sigma$ are two (unitary) representations of a second countable locally compact group $G$ of type I on Hilbert spaces $U$ and $V$ respectively. We are interested in which pairs transformations $t_U \in \mathcal{L}(U)$ and $t_V \in\mathcal{L}(V)$ are such that $t_V L = Lt_U$ for all $(\rho,\sigma)$-equivariant linear maps $L:U\to V$. The main idea is to pass to quasi-equivalent subrepresentations of $\rho$ and $\sigma$ which characterize the equivariant linear maps. We then apply a result giving a unique decomposition of type I representations in terms of direct integrals of irreducible representations. We obtain a result quite similar to the semisimple case, in terms of all ``nice'' operators on spaces (rather than all $\End(W_k)$ for irreducible representations $\rho_k$) and von Neumann algebras (rather than the group algebras $\C[\rho_k(G)]$).

The first step involves two applications of the following result.
\begin{lemma}
For any representations $\rho,\sigma$, there is a decomposition $\sigma = \sigma_{\lesssim} \oplus \sigma_{\perp}$ where $\sigma_{\lesssim}$ is subordinate to $\rho$ and $\sigma_\perp$ disjoint from $\rho$.
\end{lemma} 
\begin{proof}
Let $\mathcal{F}$ be the partially ordered set of subrepresentations of $\sigma$ which are disjoint from $\rho$, ordered by inclusion (i.e.\ being a subrepresentation). For any chain $C\subseteq \mathcal{F}$ the direct sum $\bigoplus_{\pi \in C} \pi$ contains all the representations in $C$ and is in $\mathcal{F}$, that is, is disjoint from $\rho$ \cite[Corollary 1.A.7]{Bekka2019}. By Zorn's lemma, there is a maximal subrepresentation $\sigma_\perp \leq \sigma$ disjoint from $\rho$. Let $V_\perp$ be the subspace of $V$ corresponding to $\sigma_\perp$, and $\sigma_\lesssim$ be the subrepresentation of $\sigma$ defined by restriction to the orthogonal complement of $V_\perp$. It remains to show $\sigma_\lesssim$ is subordinate to $\rho$. Suppose not, so there exists a non-trivial subrepresentation $\sigma_\lesssim'$ of $\sigma_\lesssim$ such that $\sigma_\lesssim' \perp \rho$. Then $\sigma_\lesssim' \oplus \sigma_\perp$ is a subrepresentation of $\sigma$ which is disjoint from $\rho$ (again by \cite[Corollary 1.A.7]{Bekka2019}). But $\sigma_\lesssim'$ is non-trivial, so $\sigma_\lesssim' \oplus \sigma_\perp$ is strictly larger than $\sigma_\perp$, which contradicts the maximality of $\sigma_\perp$.
\end{proof}

We thus write $\sigma = \sigma_{\lesssim} \oplus \sigma_{\perp}$. Next, decompose $\rho = \rho_{\lesssim} \oplus \rho_\perp$ such that $\rho_{\lesssim}$ is subordinate to $\sigma_\lesssim$, and $\rho_\perp$ disjoint from $\sigma_\lesssim$. Note that by construction $\rho_\lesssim$ and $\sigma_\lesssim$ are quasi-equivalent. Having passed to quasi-equivalent subrepresentations, it remains to apply some unique decomposition results. Note that both $\rho_\lesssim$ and $\sigma_\lesssim$ are type I, being subrepresentations of type I representations \parencite[Remark 6.A.13 (7)]{Bekka2019}. By \cite[Theorem 6.B.17]{Bekka2019} there exist decompositions $U = \bigoplus_{i \in I} U_i$ and $V = \bigoplus_{j\in J} V_j$ into $G$-invariant closed subspaces such that the subrepresentations of $\rho_\lesssim$ on distinct $U_i$ are disjoint, and similarly for subrepresentations of $\sigma_\lesssim$ on distinct $V_j$, and each of these subrepresentations is equivalent to a ``multiple'' (in terms of direct sums) of a multiplicity-free representation. Explicitly, $\rho_\lesssim = \bigoplus_{i\in I} \rho_i^{\oplus i}$ and $\sigma_\lesssim = \bigoplus_{j\in J} \sigma_j^{\oplus j}$, where $\rho_i$ and $\sigma_j$ are the multiplicity-free representations. Furthermore, these decompositions are unique up to permutations of the index sets $I$ and $J$. By \cite[Proposition 6.A.4]{Bekka2019}, we thus have $\rho_\lesssim \approx \bigoplus_{i \in I} \rho_i$  where the component representations $\rho_i$ are multiplicity-free and disjoint, and similarly $\sigma_\lesssim\approx \bigoplus_{j \in J} \sigma_j$. A direct sum of disjoint multiplicity-free representations is multiplicity-free  \parencite[Corollary 6.B.11]{Bekka2019}, and quasi-equivalent multiplicity-free representations are equivalent \parencite[Corollary 6.A.14]{Bekka2019}. Thus $\bigoplus_{i \in I} \rho_i$ and $ \bigoplus_{j \in J} \sigma_j$ are equivalent. By the uniqueness of the decompositions of $\rho_\lesssim$ and $\sigma_\lesssim$, we have that $I=J$ and there exists a bijection $s:I\to J$ such that $\rho_i$ is equivalent to $\sigma_{s\inv(j)}$. We denote each subrepresentation $\rho_i$ by $\pi_i$ to emphasize it is common to $\rho$ and $\sigma$.

By \cite[Proposition 6.D.5]{Bekka2019}, for each $i \in $I there is an essentially unique measure $\mu_i$ on the space $\widehat{G}$ of equivalence classes of irreducible (unitary) representations of $G$, such that
$$\pi_i = \int_{\widehat{G}}^\oplus \tau \mu_i(d\tau).$$ 

With the decomposition above, we characterize the equivariant linear maps $L:U\to V$. First note that these maps $L$ correspond to the equivariant linear maps $L:U_\lesssim \to V_\lesssim$, where the two Hilbert spaces are those on which the corresponding subrepresentations are defined: as in the semisimple case, equivariant maps $L$ must send $U_\perp$ to zero and have zero as their image projected onto $V_\perp$. The equivariant maps between $L:U_\lesssim \to V_\lesssim$, by the disjointness of the direct sum decompositions, themselves decompose as $L = \bigoplus_{i \in I} L_i$ where each $L_i$ is an equivariant map $U_i^{\oplus i} \to U_{s(i)}^{\oplus s(i)}$. These are $i\times s(i)$ ``matrices'' with entries which are $\pi_i$-equivariant linear maps $U_i \to U_i$, that is, maps in $\pi_i(G)'$. It turns out these maps are the diagonalizable operators \parencite[proof of Theorem 6.B.18]{Bekka2019}.\footnote{Explicitly, $\pi_i(G)'$ is abelian since $\pi_i$ is multiplicity-free, and by \cite[Proposition 1.G.7]{Bekka2019} the maximal abelian subalgebra of $\pi_i(G)'$, with which $\pi_i(G)'$ coincides, consists of the diagonalizable operators.}

Finally, the pairs of $t_U \in \mathcal{L}(U)$ and $t_V \in\mathcal{L}(V)$ such that $t_V L = Lt_U$ for all $L$ as above, by the same argument as in the semisimple case, are given up to the degeneracy arising from the disjoint parts $\rho_\perp$ and $\sigma_\perp$ by maps acting on each $U_i$ which commute with any diagonalizable $L_i : U_i \to U_i$. By \cite[Theorem 1.H.1]{Bekka2019}, these are the decomposable operators $U_i \to U_i$. Furthermore, by \cite[Theorem 1.I.7]{Bekka2019} and \cite[Theorem 1.I.6]{Bekka2019} we have that since $\pi_i(G)'$ consists of the diagonalizable operators,
\begin{align*}
\pi_i(G)' = \int_{\widehat{G}}^\oplus \tau(G)' \mu_i(d\tau),
\end{align*}
and
\begin{align*}
\pi_i(G)'' = \int_{\widehat{G}}^\oplus \tau(G)'' \mu_i(d\tau).
\end{align*}
In this sense, the symmetry non-uniqueness is given by taking the von Neumann algebra. In particular, one takes the von Neumann algebra of the multiplicity-free subrepresentation common to $\rho$ and $\sigma$ if these are quasi-equivalent.

\section{Future work}
\label{sec-future}
Our present work leaves open several avenues. Most glaringly, we do not give a constructive answer to how one might learn symmetries. Rather, we hope this work provides a valuable framework for formalizing the question and understanding why certain approaches fail.

It is worth noting the role that priors or ``implicit biases'' may play in learning symmetry. The theory presented here describes when a symmetry may not be identifiable, but by imposing a certain bias the correct symmetry may still be recovered. Despite not exhibiting symmetry non-uniqueness, the recent works of \textcite{Zhou2021, Dehmamy2021} may succeed in part due to non-trivial biases imposed by restricted parameterizations of learnable symmetries.

Another area we believe is likely important to understand is that of approximate symmetry. There are many reasons for this, some more obvious than others. One common justification is that real-world data is noisy, and thus even if an underlying system displays exact symmetry, approximate symmetry may better serve attempts to model data; a related response to stochasticity is to study ``probabilistic symmetry,'' as done by \textcite{BloemReddy2020}. There is of course also the standard observation that while small transformations may constitute a symmetry, large ones may not: ``6'' and ``9'' are not the same. In the spirit of the current work, we add another possible motivation: continuity.  This is illustrated by a cute example demonstrating a difficultly in simultaneously learning an invariant function and its symmetry using a strictly invariant ansatz.
\begin{example}
The group of rotation matrices $R_t$ for $t \in \R$ is generated by the matrix for a rotation by $\pi/2$, in the sense that taking matrix exponentials we have
\begin{align*}
&R_t = \exp(tA) = \begin{pmatrix} \cos(t) & - \sin(t) \\ \sin(t) & \cos(t) \end{pmatrix}, &&A=\begin{pmatrix} 0 & -1 \\ 1 & 0 \end{pmatrix}.
\end{align*}
Suppose we are learning $A$ by approximation, with learnable $\epsilon\in\R$ parameterizing
$$\widetilde{A}(\epsilon)=\begin{pmatrix} \epsilon & -1 \\ 1 & \epsilon \end{pmatrix}.$$
Then the generated matrix group consists of matrices of the form
$$\widetilde{R_t}=\exp(t\widetilde{A}(\epsilon)) = e^{\epsilon t} \begin{pmatrix} \cos(t) & - \sin(t) \\ \sin(t) & \cos(t) \end{pmatrix}.$$
As a function of $\epsilon$, the orbit of the first basis vector $(1,0)$ is
$$\mathcal{O}(\epsilon)=\{e^{\epsilon t}(\cos(t),\sin(t)):t\in\R\}.$$
When $\epsilon = 0$ this is the unit circle. For $\epsilon>0$, however, this is a spiral. In particular $\mathcal{O}(\epsilon)$ converges to all of $\R^2$ as $\epsilon\to 0$: for any $x\in\R^2$ and $\delta > 0$, there is an $\epsilon\in \R$ small enough such that $||x-y|| < \delta$ for some $y\in\mathcal{O}(\epsilon)$. An ansatz of strictly $\widetilde{R_t}$-invariant functions will thus be forced towards constant functions as $\epsilon\to 0$, and will not be able to recover general rotation-invariant functions.
\end{example}
Similar examples can be concocted for discrete groups. One may hope that this problem is particular to invariance, rather than general equivariance, but continuity issues can be shown to arise in the more general context as well, in particular by formulating equivariance as invariance in a tensor product space (as in Example \ref{ex-invariant-polynomials}). We thus see understanding approximate symmetry as a potentially important part of making symmetries learnable.

\section*{Acknowledgements}
\addcontentsline{toc}{section}{Acknowledgements}
We thank Adam Quinn Jaffe and Brad Ross for helpful discussions. We also thank Marc Finzi, for pointers on running experiments using the framework of \parencite{EMLP}.

\appendix
\section{Proof of Theorem \ref{thm-gcnn-uniqueness}}
\label{appendix-conv-uniqueness}
Recall that $X,Y$ are Hausdorff Borel spaces with strictly positive measures $\mu,\nu$, and $G$ is a compact Hausdorff group with Haar measure $\lambda$. Assuming that any $G$-equivariant integral $L:L^1(X,\mu)\to L^\infty(Y,\nu)$ is $T$-equivariant and $\mu$ is both $G$-invariant and $T$-invariant, we prove that for any non-empty open $A\subseteq G$ and $t \in T$ we have $c_t \overline{A} = t(\overline{A})$ where $c_t = t(\id)$. The proofs of $\overline{A} = t(\overline{A})$ for non-empty open $A\subseteq X/G$ and $A\subseteq Y/G$ follow the same formal structure, with $c_t$ replaced by an identity map (recalling that the quotients of Hausdorff spaces by compact groups are Hausdorff).

We first show that for any measurable $A \subseteq G$ and any $t \in T$,
$$\lambda\{c_t\inv g \in A, ~t\inv(g) \not\in A\} = \lambda\{c_t\inv g \not\in A, ~t\inv(g) \in A\} = 0.$$
The first set is just $c_t A \setminus t(A)$ and the second $t(A)\setminus c_t A$, both of which are subsets of $c_t A \triangle t(A)$, where $\triangle$ denotes the symmetric difference operator. It thus suffices to show $\lambda(c_t A \triangle t(A))= 0$. Let $\ell(h\inv g, o, p) =\ind{h\inv g \in A}$. The integral operator $L$ given by
$$(Lf)(h,p) = \int\int \ell(h\inv g, o, p) f(g,o) \lambda(dg) \mu_{X/G}(do) \le \int f(x) \mu(dx)$$
is by construction a $G$-equivariant operator from $L^1(X,\mu)$ to $L^\infty(Y,\nu)$. For any $t \in T$, by assumption $L$ is equivariant under $t$, which by Corollary \ref{cor-inv-kernel-function} implies that 
$$\ind{t(h)\inv g \in A} = \ind{h\inv t\inv(g) \in A}$$
for $\mu$-a.e. $(g,o)\in X$. Fixing $h = \id$, with some rearranging we have
$$\ind{g \in c_tA} = \ind{ g \in t(A)}$$
for $\mu$-a.e. $(g,o)\in X$, or equivalently $\lambda(c_t A \triangle t(A))= 0$.

Suppose now for the sake of contradiction that $A\subseteq G$ is a non-trivial open set, but $c_t \overline{A} \ne t(\overline{A})$. That is, there exists $g \in c_t \overline{A}\setminus t(\overline{A})$ or $g \in t(\overline{A})\setminus c_t\overline{A}$. Without loss of generality, we consider the former case. That is, $c_t\inv g \in \overline{A}$ but $t\inv(g) \not \in \overline{A}$. Note that any neighborhood of $c_t\inv g$ intersects $A$ non-trivially. Next, observe that $c_t\inv t(\overline{A}^c)$ contains $c_t\inv g$, and furthermore is open, since $G$ and $T$ act by homeomorphisms. Thus, $c_t\inv t(\overline{A}^c) \cap A$ is a non-empty open set, and therefore has strictly positive measure under $\lambda$ (which is strictly positive, being a factor of $\mu$).\footnote{Haar measures are in fact always strictly positive, but the above reasoning generalizes to $\mu_{X/G}$.} On the other hand,
$$c_t \inv t(\overline{A}^c) = \{c_t\inv g: t\inv(g)\not\in\overline{A}\} \subseteq \{c_t\inv g : t\inv(g) \not\in A\}$$
and thus $c_t\inv t(\overline{A}^c)\cap A$ is a subset of $\{c_t\inv g\in A,~t\inv(g)\not\in A\}$ and must have zero measure by the previous paragraph. We therefore arrive at a contradiction, so $c_t \overline{A} = t(\overline{A})$. \qed

\newpage
\emergencystretch=1em
\printbibliography[heading=bibintoc]
\end{document}